\documentclass[11pt]{article}
\usepackage{amsmath,amsfonts,amssymb,mathrsfs}
\usepackage{graphicx,psfrag,epsf,wrapfig}
\usepackage{enumerate,color,centernot}
\usepackage[authoryear]{natbib}
\usepackage{adjustbox,booktabs,arydshln,multirow}
\usepackage{caption}
\usepackage{subcaption}
\usepackage{xcolor}
\usepackage[hyphens]{url}
\usepackage[breaklinks=true]{hyperref}

\usepackage[ruled,vlined]{algorithm2e}
\usepackage{graphicx}
\newcommand{\indep}{\rotatebox[origin=c]{90}{$\models$}}

\usepackage[a4paper,margin=1.0in]{geometry}
\usepackage{xr}
 
\newcommand{\bX}{{\boldsymbol X}}
\newcommand{\bY}{{\boldsymbol Y}}
\newcommand{\by}{{\boldsymbol y}}
\newcommand{\bx}{{\boldsymbol x}} 
\newcommand{\bz}{{\boldsymbol z}}

\newcommand{\bZ}{{\boldsymbol Z}}

\newcommand{\be}{{\boldsymbol e}}
\newcommand{\bw}{{\boldsymbol w}}

\newcommand{\ba}{{\boldsymbol a}}

\newcommand{\mK}{\mathcal{K}}
\newcommand{\bV}{{\boldsymbol V}}

\newcommand{\mG}{\mathcal{G}}

\newcommand{\mX}{\mathcal{X}}
\newcommand{\mM}{\mathcal{M}}
\newcommand{\bU}{{\boldsymbol U}}
\newcommand{\bu}{{\boldsymbol u}}

\newcommand{\bs}{{\boldsymbol s}}
\newcommand{\bS}{{\boldsymbol S}}

\newcommand{\Var}{{\rm Var}}
\newcommand{\diag}{{\rm Diag}}

\newcommand{\bxi}{{\boldsymbol \xi}}

\newcommand{\bbeta}{{\boldsymbol \beta}}

\newcommand{\btheta}{{\boldsymbol \theta}}

\newcommand{\bepsilon}{{\boldsymbol \epsilon}}

\newcommand{\bSigma}{{\boldsymbol \Sigma}}

\newtheorem{theorem}{Theorem}[section]
\newtheorem{lemma}{Lemma}[section]

\newtheorem{remark}{Remark}
\newtheorem{assumption}{Assumption}
\newenvironment{proof}{\trivlist\item[\hskip \labelsep{\sc Proof:}]}
 {\unskip\nobreak\ \lower.3ex\hbox{$\Box$}\endtrivlist}

\begin{document}
 

\title{Extended Fiducial Inference for  Individual Treatment Effects via Deep Neural Networks} 

\author{Sehwan Kim$^{\dag}$ and Faming Liang\thanks{Correspondence author: Faming Liang, email: fmliang@purdue.edu.
 $^\dag$ \textcolor{black}{Department of Statistics, Ewha Womans University, Seoul 03760, Republic of Korea.}
  $^*$ Department of Statistics, Purdue University, West Lafayette, IN 47907, USA.
 } }
 

\date{} 

\date{\today}
 
\maketitle

\begin{abstract}
Individual treatment effect estimation has gained significant attention in recent data science literature. This work introduces the Double Neural Network (Double-NN) method to address this problem within the framework of extended fiducial inference (EFI). In the proposed method, deep neural networks are used to model the treatment and control effect functions, while an additional neural network is employed to estimate their parameters. The universal approximation capability of deep neural networks ensures the broad   applicability of this method. Numerical results highlight the superior  performance of the proposed Double-NN method compared to the conformal quantile regression (CQR) method in individual treatment effect estimation. From the perspective of statistical inference, this work advances the theory and methodology for statistical inference of large models. Specifically, it is theoretically proven that the proposed method permits the model size to increase with the sample size $n$ at a rate of $O(n^{\zeta})$ for some  $0 \leq \zeta<1$, while still maintaining proper quantification of uncertainty in the model parameters. This result marks a  significant improvement compared to the range $0\leq \zeta < \frac{1}{2}$ required by the classical central limit theorem.  Furthermore, this work provides a rigorous framework for quantifying the uncertainty of deep neural networks under the neural scaling law, representing a substantial contribution to the statistical understanding of large-scale neural network models.
\end{abstract}


\section{Introduction}

Causal inference is a fundamental problem in many disciplines such as medicine, econometrics, and social science. Formally, let $\{(y_1,\bx_1,t_1), (y_2,\bx_2,t_2),\ldots, (y_n,\bx_n,t_n)\}$ denote a set of observations drawn from the following data-generating equations: 
\begin{equation} \label{dataGeq}
y_i=c(\bx_i) +\tau(\bx_i) t_i+\sigma z_i,  \quad i=1,2,\ldots,n,
\end{equation} 
where $\bx_i \in \mathbb{R}^d$ represents a vector of covariates of subject $i$, 
$t_i \in \{0,1\}$ represents the treatment assignment to subject $i$; 
$c(\cdot)$ represents the expected outcome of 
subject $i$ if assigned to the control group (with $t_i=0)$, and $\tau(\bx_i)$ is the 
expected treatment effect of subject $i$ if assigned to the treatment group (with $t_i=1$); $\sigma>0$ is the standard deviation,  and $z_i$ represent a standardized  random error that is not necessarily Gaussian. Under the potential outcome framework \citep{Rubin1974EstimatingCE}, each individual  
receives only one assignment of the treatment with $t_i=0$ or 1, but not both. 
The goal of causal inference is to make inference for the average treatment effect (ATE) or individual treatment effect (ITE). 

The ATE is defined as 
\begin{equation}\label{ATEeq}
\tau_0=\mathbb{E}(\tau(\bx))=\int_{\mX} \tau(\bx) dF(\bx),
\end{equation}
where $\mX$ denotes the sample space of $\bx$, and 
$F(\bx)$ denotes the cumulative distribution function of $\bx$.
 To estimate ATE, a variety of methods, including outcome regression, augmented/inverse probability weighting (AIPW/IPW) and matching, have been developed. See \cite{Imbens2004NonparametricEO} and \cite{Rosenbaum2002Book} for overviews. 
 
The ITE is often defined as the conditional average treatment effect (CATE):   
\begin{equation} \label{CATEeq}
\tau(\bx)=\mathbb{E}(Y|T=1,\bx)-\mathbb{E}(Y|T=0,\bx),
\end{equation}
see e.g., \cite{shalit2017pehe} and \cite{Lu2018EstimatingIT}. 
Recently, \cite{lei2021ite} proposed to make 
predictive inference of the ITE by quantifying the uncertainty of 
\begin{equation} \label{ITEeq}
\tilde{\tau}_i:=Y(T=1,\bx_i)-Y(T=0,\bx_i):=Y_i(1)-Y_i(0),
\end{equation}
where $Y_i(t_i)$ denotes the potential outcome of subject  
$i$ with treatment assignment $t_i \in \{0,1\}$. Henceforth, we will call $\tilde{\tau}_i$ the predictive ITE.  

It is known that ATE and ITE are identifiable if 
the conditions {\it `strong ignorability'} and {\it `overlapping'} are satisfied. \textcolor{black}{The former means that, after accounting for observed covariates, the treatment assignment is   
independent of potential outcomes;
and the latter ensures that every subject in the study has a positive 
probability of receiving either assignment, allowing for meaningful comparisons between treatment and control groups.}  
Mathematically, the two conditions can be expressed as: 
\[
(i) \ \mbox{\it strong ignorability:} 
\ \{Y(1),Y(0)\} \indep T|\bx; \ \ 
(ii) \ \mbox{\it overlapping:} \   0<P(T=1|\bx)<1, \ \  \forall \bx \in \mX,
\]
where $T\in \{0,1\}$ represents the treatment assignment variable, and $\indep$ denotes conditional independence.  
Together, they ensure that the causal effect can be 
correctly estimated without bias.
See e.g. \cite{Guan2019AUF} for more discussions on this issue. 

However, even under these assumptions, accurate inference for ATE and ITE can still be challenging. 
Specifically, the inference task can be complicated by unknown nonlinear forms 
of $c(\bx)$ and $\tau(\bx)$.
To address \textcolor{black}{these} issues, some authors have proposed to approximate them 
using a machine learning model, such as random forest (RF) \citep{breiman2001random}, Bayesian additive regression trees (BART) \citep{Chipman2010BARTBA}, and neural networks.  Refer to e.g.,
\cite{Foster2011SubgroupIF}, \cite{Hill2011bart},  \cite{shalit2017pehe},   \cite{wager2018nonpar2}, and \cite{hahn2020nonpar3} for the \textcolor{black}{details}. 
\textcolor{black}{Unfortunately, these methods often yield point estimates for  the ATE and ITE, while failing to correctly quantifying their uncertainty due to the complexity of the machine learning models.} Quite recently, \cite{lei2021ite} proposed to quantify the uncertainty of the predictive ITE  using the conformal inference method \citep{Vovk2005AlgorithmicLI,Shafer2008}. This method provides coverage-guaranteed confidence intervals for 
the predictive ITE,  but the intervals may become overly wide when the machine learning model is not consistently estimated. \textcolor{black}{In short, 
while machine learning models, particularly neural networks, can effectively model complex, nonlinear functions such as $c(\cdot)$ and $\tau(\cdot)$ for causal inference,  performing accurate uncertainty quantification with 
these models remains a significant challenge.
This is because these models typically have a complex functional form and involve a large number of parameters.} 
 
 In this paper, we propose to conduct causal inference using an extended fiducial inference (EFI) method \citep{LiangKS2024EFI}, with the goal of addressing the uncertainty quantification issue associated with treatment effect estimation.  
{\color{black} EFI provides an innovative framework for inferring model uncertainty based solely on observed data, aligning with the goal of fiducial inference \citep{Fisher1935The,hannig2009gfi}. 
Specifically, it aims to solve the data-generating equations by explicitly imputing the unobserved random errors and approximating the model parameters from the observations and imputed random errors using a neural network; it then infers the uncertainty of the model parameters based on the learned neural network function and the imputed random errors (see Section \ref{sect:EFI} for a brief review).}
To make the EFI method feasible for causal effect estimation with accurate uncertainty quantification, we extend the method in two key aspects:
\begin{itemize} 
\item[(i)]  We approximate each of the unknown functions, $c(\bx)$ and $\tau(\bx)$, by a deep neural network (DNN) model. 
\textcolor{black}{The DNN possesses universal approximation capability \citep{HornikSW1989,
Hornik1991ApproximationCO,Kidger2020UniversalAW}, meaning it can approximate any continuous function to an arbitrary degree of accuracy, provided it is sufficiently wide and deep.} This property makes the proposed method applicable to a wide range of data-generating processes.
 

\item[(ii)] We theoretically prove that the dimensions (i.e., the number of parameters) 
of the DNN models used to approximate $c(\bx)$ and $\tau(\bx)$ are allowed to increase with the sample size $n$ at a rate of $O(n^{\zeta})$ for some $0< \zeta<1$, while the uncertainty of the DNN models can still be correctly quantified. That is, we are able to correctly quantify 
the uncertainty of the causal effect although it has to be approximated using large models. 
\end{itemize} 

In this paper, we regard a model as `large' if its dimension increases with $n$ at a rate of $1/2\leq \zeta < 1$. 
We note that part (ii) represents a significant theoretical innovation in  statistical inference for large models.  
 In the literature on this area,  most efforts have focused on linear models, featuring techniques such as desparsified Lasso \citep{Javanmard2014,vandeGeer2014,ZhangZhang2014}, post-selection inference \citep{LeeTaylor2016}, 
 and Markov neighborhood regression \citep{LiangXJ2020MNR}.
For nonlinear models, the research landscape appears to be more scattered. 
\cite{Portnoy1986OnTC,Portnoy1988} showed that for independently and identically 
distributed (i.i.d) random vectors with the dimension $p$ increasing with the sample 
size $n$, the central limit theorem (CLT) holds if $p=O(n^{\zeta})$ for some $0\leq \zeta<1/2$. 
It is worth noting that Bayesian methods, despite being sampling-based, do 
not permit the dimension of the true model to increase with $n$ 
at a higher rate. For example, even in the case of generalized linear models, 
to ensure the posterior consistency, the dimension of the true model 
is only allowed to increase with $n$ at a rate $0 \leq \zeta< 1/4$ (see Theorem 2 and Remark 2 of \cite{jiang2007bayesian}).  
Under its current theoretical framework developed by \cite{LiangKS2024EFI}, 
EFI can only be applied to make inference for 
the models whose dimension is fixed or increases with $n$ at a very low rate. 
This paper extends the theoretical framework of EFI further, 
establishing its applicability for statistical inference of large models. 

It is worth noting that a DNN model with size $p=O(n^{\zeta})$, where $\zeta$ is close to (but less than) 1,  
has been shown to be sufficiently large for approximating many data generation processes. 
This is supported by the theory established in \cite{SunSLiang2021} and \cite{Farrell2021DeepNN}.
 In \cite{SunSLiang2021}, it is shown that, 
 as $n \to \infty$,
 a sparse DNN model of this size 
can provide accurate approximations for multiple classes 
of functions, such as bounded $\alpha$-H\"older smooth functions \citep{Schmidt-Hieber2017Nonparametric},  piecewise smooth functions with fixed input dimensions \citep{petersen2018optimal}, and functions representable  
by an affine system \citep{bolcskei2019optimal}. 
Similar results have also been obtained in 
 \cite{Farrell2021DeepNN}, where it is shown that a  
multi-layer perceptron (MLP) with this model size and the ReLU activation function   
can provide an accurate approximation to the functions that lie in a Sobolev ball with certain smoothness.  
The approximation capability of DNNs of this size has also been empirically 
validated by \cite{Hestness2017DeepLS},  where a neural 
scaling law of $p =O(n^{\zeta})$ with $0.5 \leq \zeta <1$ was 
identified through extensive studies across various model architectures in machine translation, language modeling, image processing, and speech recognition.
 
To highlight the strength of EFI in uncertainty quantification and to facilitate comparison with the conformal inference method, this study focuses on inference for predictive ITEs, although the proposed method can also be extended to ATE and CATE. Our numerical results demonstrate the superiority of the proposed method over the conformal inference method.

The remaining part of this paper is organized as follows. Section 2 provides a brief review
of the EFI method. Section 3 extends EFI to statistical inference for large statistical models.  Section 4 provides an illustrative example for EFI. 
Section 5 applies the proposed method to 
statistical inference for predictive ITEs, with both simulated and real data examples. Section 6 concludes the paper with a brief discussion.

 \section{A Brief Review of the EFI Method} \label{sect:EFI}


While fiducial inference was widely considered as a big blunder by R.A. Fisher, the goal he initially set ---inferring the uncertainty of model parameters on the basis of observations --- has been continually pursued by many statisticians, see e.g. \cite{Zabell1992Fisher},  \cite{hannig2009gfi}, \cite{hannig2016gfi},  \cite{Murph2022GeneralizedFI}, and \cite{Martin2023FiducialIV}.
To this end, \cite{LiangKS2024EFI} 
developed the EFI method based on the fundamental concept of structural inference 
\citep{Fraser1966StructuralPA,Fraser1968Book}. 
Consider a regression model: 
\begin{equation} \label{modeleq}
Y=f(\bX,Z,\btheta), 
\end{equation}
where $Y\in \mathbb{R}$ and $\bX\in \mathbb{R}^{d}$ represent the response and explanatory variables, respectively; $\btheta\in \mathbb{R}^p$ represents the vector of  parameters; 
and $Z\in \mathbb{R}$ represents a scaled random error following  
 a known distribution $\pi_0(\cdot)$.  
 For the model (\ref{dataGeq}),  the treatment assignment $T$ should be included as a part of $\bX$. 

Suppose that a random sample of size $n$ has been collected from the model, denoted by $\{(y_1,\bx_1), (y_2,\bx_2),\ldots,(y_n,\bx_n)\}$. In the point of view of structural inference \textcolor{black}{ \citep{Fraser1966StructuralPA,Fraser1968Book}}, they can be expressed in the data generating equations as follow:  
\begin{equation} \label{dataGeneqg}
y_i=f(\bx_i,z_i,\btheta), \quad i=1,2,\ldots,n.
\end{equation}
This system of equations consists of $n+p$ unknowns, namely, $\{\btheta, z_1, z_2, \ldots, z_n
\}$, while there are only $n$ equations. Therefore, the values of $\btheta$ cannot be uniquely determined by the data-generating equations, and this lack of uniqueness of unknowns introduces uncertainty in $\btheta$.  

Let $\bZ_n=\{z_1,z_2,\ldots,z_n\}$ denote the unobservable random errors,
which are also called latent variables in EFI.  
Let $G(\cdot)$ denote an inverse function/mapping for the parameter $\btheta$, i.e., 
\begin{equation} \label{Inveq}
\btheta=G(\bY_n,\bX_n,\bZ_n).
\end{equation}
It is worth noting that the inverse function is generally non-unique. For example, it can be constructed by solving any $p$ equations in (\ref{dataGeneqg}) for $\btheta$. 
As noted by \cite{LiangKS2024EFI}, this non-uniqueness of inverse function 
mirrors the flexibility of frequentist methods, where different 
estimators of $\btheta$ can be designed for different purposes. 

As a general method, \cite{LiangKS2024EFI} proposed to approximate the inverse function $G(\cdot)$ using a sparse DNN, see Figure \ref{EFInetwork} for illustration.
They also introduced an adaptive stochastic gradient Langevin dynamics (SGLD) algorithm, which facilitates the  simultaneous 
training of the sparse DNN 
 and simulation of the latent variables $\bz$. This is briefly described  
 as follows.

\begin{figure}[htbp]
    \centering
    \includegraphics[width=0.8\textwidth]{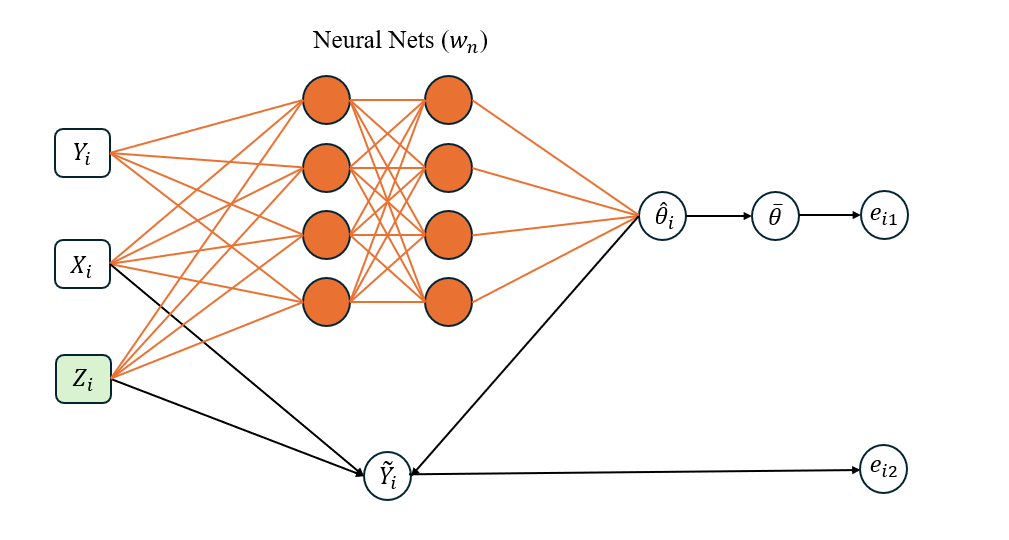}
    \caption{Illustration of the EFI network \citep{LiangKS2024EFI}, where the orange nodes and orange links form a DNN (parameterized by the weights $\bw_n$, with the subscript $n$ indicating its dependence on the training sample size $n$), the green node represents latent variable to impute, and the black lines represent deterministic functions. }
    \label{EFInetwork}
\end{figure}

Let $\hat{\btheta}_i:=\hat{g}(y_i,\bx_i,z_i,\bw_n)$ denote the DNN prediction function parameterized by the weights $\bw_n$ in the EFI network, and let 
\begin{equation} \label{thetabareq}
\bar{\btheta}:=\frac{1}{n} \sum_{i=1}^n \hat{\btheta}_i=\frac{1}{n} \sum_{i=1}^n \hat{g}(y_i,\bx_i,z_i,\bw_n),
\end{equation}
which serves as an estimator of $G(\cdot)$.  
The EFI network has two output nodes defined, respectively, by 
\begin{equation} \label{outputeq}
e_{i1} :=\|\hat{\btheta}_i-\bar{\btheta}\|^2, \quad 
e_{i2} :=d(y_i,\tilde{y}_i):=d(y_i,\bx_i, z_i, \bar{\btheta}), 
\end{equation}
where $\tilde{y}_i=f(\bx_i,z_i,\bar{\btheta})$,  $f(\cdot)$ is as specified in (\ref{dataGeneqg}), and $d(\cdot)$ is a function that measures the difference between $y_i$ and $\tilde{y}_i$. For example, for a normal linear/nonlinear regression, 
it can be defined as 
\begin{equation} \label{deq0}
d(y_i,\bx_i,z_i,\bar{\btheta})=\|y_i-f(\bx_i,z_i,\bar{\btheta})\|^2.
\end{equation}
For logistic regression, it is defined as a squared ReLU function, see  \cite{LiangKS2024EFI} for the \textcolor{black}{details}. 
Furthermore, EFI defines an energy function as follows: 
\begin{equation} \label{energyfunction11}
 U_n(\bY_n,\bX_n,\bZ_n,\bw_n) =  \sum_{i=1}^n d(y_i,\bx_i,z_i,\bar{\btheta}) + 
 \eta \sum_{i=1}^n\| \hat{\btheta}_i- \bar{\btheta} \|^2, 
\end{equation}
for some regularization coefficient $\eta>0$,
\textcolor{black}{where first term measures the fitting error of the model as implied by equation (\ref{deq0}),
and the second term regularizes the variation of 
$\hat{\btheta}_i$, ensuring that the neural network forms a proper estimator of the inverse function.} 
Given this energy function, we define the likelihood function as
\begin{equation} \label{likelihoodeq} 
\pi_{\epsilon}(\bY_n|\bX_n,\bZ_n,\bw_n) \propto  e^{- U_n(\bY_n,\bX_n,\bZ_n,\bw_n)/\epsilon},
\end{equation}
for some constant $\epsilon$ close to 0. 
As discussed in \cite{LiangKS2024EFI}, the choice of $\eta$ does not \textcolor{black}{have much affect} on the performance of EFI as long as 
$\epsilon$ is sufficiently small. 

Subsequently, the posterior of $\bw_n$ is given by 
\begin{equation} \label{eqA12}
\begin{split}
\pi_{\epsilon}(\bw_n|\bX_n,\bY_n,\bZ_n) &  
\propto \pi(\bw_n) e^{-U_n(\bY_n,\bX_n,\bZ_n,\bw_n)/\epsilon }, 
 \end{split} 
\end{equation} 
where $\pi(\bw_n)$ denotes the prior  of $\bw_n$;
and the predictive distribution of $\bZ_n$ is given by  
\begin{equation} \label{eqA13}
\begin{split}
 \pi_{\epsilon}(\bZ_n|\bX_n,\bY_n,\bw_n) &  
 \propto \pi_0^{\otimes n}(\bZ_n) e^{-U_n(\bY_n,\bX_n,\bZ_n,\bw_n)/\epsilon}.
 \end{split} 
\end{equation}

 In EFI, $\bw_n$ is estimated through maximizing the posterior 
 $\pi_{\epsilon}(\bw_n|\bX_n,\bY_n)$ given the observations $\{\bX_n,\bY_n \}$.
  By the Bayesian version of Fisher's identity \citep{SongLiang2020eSGLD}, 
 the gradient equation $\nabla_{\bw_n} \log \pi_{\epsilon}(\bw_n|\bX_n,\bY_n)$ $=0$ can be re-expressed as
 \begin{equation} \label{identityeq}
\nabla_{\bw_n} \log \pi_{\epsilon}(\bw_n|\bX_n,\bY_n)=\int  \nabla_{\bw_n} \log \pi_{\epsilon}(\bw_n|\bX_n,\bY_n,\bZ_n) \pi_{\epsilon}(\bZ_n|\bX_n,\bY_n,\bw_n) d\bw_n=0,
\end{equation} 
which can be solved using  an adaptive stochastic gradient MCMC algorithm 
\citep{LiangSLiang2022,deng2019adaptive}. 
The algorithm 
works by iterating between  two steps: 
\begin{itemize} 
\item[(a)] {\it Latent variable sampling}: draw $\bZ_n^{(k+1)}$ according to a Markov transition kernel that leaves $\pi_{\epsilon}(\bz|\bX_n,\bY_n,\bw_n^{(k)})$ to be invariant; 

\item[(b)] {\it Parameter updating}:  update $\bw_n^{(k)}$ toward the maximum of $\log \pi_{\epsilon}(\bw_n|\bX_n,\bY_n,\bZ_n)$ using stochastic approximation \citep{RobbinsM1951}, based on the sample $\bZ_n^{(k+1)}$.
\end{itemize}
See Algorithm \ref{EFIalgorithm} for the pseudo-code.  
 This algorithm is termed ``adaptive'' because the transition kernel in the latent variable sampling step changes with the working parameter estimate of  
 $\bw_n$. The parameter updating step can be implemented using mini-batch SGD, and the latent variable sampling step can be executed in parallel for each observation $(y_i,\bx_i)$. Hence, the algorithm is scalable with respect to  large datasets.

 \begin{algorithm}[!ht]
 \caption{Adaptive SGHMC for Extended Fiducial Inference}
 \label{EFIalgorithm}
 \SetAlgoLined 
 {\bf (i) (Initialization)} Initialize
 $\bw_n^{(0)}$, $\bZ_n^{(0)}$, $M$ (the number of fiducial samples to collect), and $\mK$ (burn-in iterations).

\For{k=1,2,\ldots,\mbox{$\mK+M$}}{

 {\bf (ii) (Latent variable sampling)} Given $\bw_n^{(k)}$, simulate $\bZ_n^{(k+1)}$ by the SGHMC algorithm \citep{SGHMC2014}: 
 \[
  \begin{split}
\bV_n^{(k+1)} &= (1-\varpi) \bV_n^{(k)} +\upsilon_{k+1} \widehat{\nabla}_{\bZ_n} \log \pi_{\epsilon}(\bZ_n^{(k)}|\bX_n,\bY_n,\bw_n^{(k)}) +\sqrt{2 \varpi \tilde{\tau}  \upsilon_{k+1}} \be^{(k+1)}, \\
\bZ_n^{(k+1)}&=\bZ_n^{(k)}+\bV_n^{(k+1)}, 
\end{split}
\]
where $\varpi$ is the moment parameter,   $\upsilon_{k+1}$ is the learning rate, 
$\tilde{\tau}=1$ is the temperature, and  $\be^{(k+1)} \sim N(0,I_{d_{\bz}})$.

 {\bf (iii) (Parameter updating)} Draw a minibatch $\{(y_1,\bx_1,z_1^{(k)}),\ldots,(y_m,\bx_m,z_m^{(k)})\}$ 
 and update the network weights by the SGD algorithm: 
\begin{equation}
\label{sgd_update}
\bw_n^{(k+1)}=\bw_n^{(k)} +\gamma_{k+1} \left[ \frac{n}{m} \sum_{i=1}^m \nabla_{\bw_n} \log \pi_{\epsilon}(y_i|\bx_i,z_i^{(k)},\bw_n^{(k)})+ \nabla_{\bw_n} \log \pi(\bw_n^{(k)}) \right],
\end{equation}
where $\gamma_{k+1}$ is the step size, and $\log \pi_{\epsilon}(y_i|\bx_i,z_i^{(k)},\bw_n^{(k)})$ can be appropriately defined according to \textcolor{black}{(\ref{likelihoodeq})}. 
  
{\bf (iv) (Fiducial sample collection)} If $k+1 > \mK$, calculate 
 $\hat{\btheta}_i^{(k+1)}=\hat{g}(y_i,\bx_i,z_i^{(k+1)},\bw_n^{(k+1)})$ 
 for each $i\in \{1,2,\ldots,n\}$ and average them to get a fiducial $\bar{\btheta}$-sample as calculated in (\ref{thetabareq}). 
}

{\bf (v) (Statistical Inference)} Conducting statistical inference for the model based on the collected fiducial samples. 
\end{algorithm} 
  

Under mild conditions for adaptive stochastic gradient 
MCMC algorithms \citep{deng2019adaptive,LiangSLiang2022}, it was shown in \cite{LiangKS2024EFI} that 
\begin{equation} \label{wconvergence}
\|\bw_n^{(k)} -\bw_n^* \|\stackrel{p}{\to} 0, \quad \mbox{as $k\to \infty$},
\end{equation}
 where $\bw_n^*$ denotes a solution to equation (\ref{identityeq}) and $\stackrel{p}{\to}$ denotes convergence in probability, and  that 
\begin{equation} \label{Zconvergence}
\bZ_n^{(k)} \stackrel{d}{\rightsquigarrow} \pi_{\epsilon}(\bZ_n|\bX_n,\bY_n,\bw_n^*), 
\quad \mbox{as $k \to \infty$},
\end{equation}
in 2-Wasserstein distance, where  $\stackrel{d}{\rightsquigarrow}$ denotes weak convergence. 

To study the limit of (\ref{Zconvergence}) as $\epsilon$ decays to 0, i.e.,  
\[
p_n^*(\bz|\bY_n,\bX_n,\bw_n^*)=  
\lim_{\epsilon \downarrow 0} \pi_{\epsilon}(\bZ_n|\bX_n,\bY_n,\bw_n^*),
\]
where 
$p_n^*(\bz|\bY_n,\bX_n,\bw_n^*)$ 
is referred to as the extended fiducial 
density (EFD) of $\bZ_n$ learned in EFI, 
it is necessary for $\bw_n^*$ to be a consistent estimator of $\bw_*$, the parameters 
of the underlying true EFI network.
To ensure this consistency, \cite{LiangKS2024EFI} impose some conditions on the structure of the DNN 
 and the prior distribution $\pi(\bw_n)$. Specifically, they assume that $\bw_n$ takes values in a compact space $\mathcal{W}$; $\pi(\bw_n)$ is a truncated mixture Gaussian distribution on  $\mathcal{W}$; and the DNN structure satisfies certain constraints given in \cite{SunSLiang2021}, e.g., the width of the output layer (i.e., the dimension of $\btheta$) is fixed or grows very slowly with $n$. They then 
 justify the consistency of $\bw_n^*$  based on the sparse 
 deep learning theory developed in \cite{SunSLiang2021}. 
 The consistency of $\bw_n^*$ further implies that 
 \[
 G^*(\bY_n,\bX_n,\bZ_n)= \frac{1}{n} \sum_{i=1}^n \hat{g}(y_i,\bx_i,z_i,\bw_n^*),
 \]
 serves as a consistent estimator for the inverse 
function/mapping $\btheta=G(\bY_n,\bX_n,\bZ_n)$.
 
By Theorem 3.2 in \cite{LiangKS2024EFI}, for the target model (\ref{dataGeq}), which is a noise-additive model, the EFD of $\bZ_n$ is 
invariant to the choice of the inverse function, provided that $d(\cdot)$ is  specified as in (\ref{deq0}) in defining the energy function. 
Further, by Lemma 4.2 in \cite{LiangKS2024EFI}, $p_n^*(\bz|\bY_n,\bX_m,\bw_n^*)$ is given by 
\begin{equation} \label{EFDeq:Z}
\frac{dP_n^*(\bz|\bX_n,\bY_n,\bw_n^*)}{d\nu}= \frac{\pi_0^{\otimes n}(\bz)}{\int_{\mathcal{Z}_n} \pi_0^{\otimes n}(\bz) d \nu}, 
\end{equation}
where $P_n^*(\bz|\bX_n,\bY_n,\bw_n^*)$ 
represents the cumulative distribution 
function (CDF) corresponding to $p_n^*(\bz|\bX_n,\bY_n,\bw_n^*)$;
$\mathcal{Z}_n=\{\bz: U_n(\bY_n,\bX_n,\bZ_n, \bw_n^*)=0\}$  
represents the zero-energy set, which forms a 
manifold in the space $\mathbb{R}^n$; and 
$\nu$ is the sum of intrinsic measures on the $p$-dimensional manifold in 
$\mathcal{Z}_n$. That is, under the consistency of $\bw_n^*$,   
$p_n^*(\bz|\bX_n,\bY_n,\bw_n^*)$ is reduced to a truncated density function  of $\pi_0^{\otimes n}(\bz)$ on the manifold 
$\mathcal{Z}_n$, while $\mathcal{Z}_n$ itself is also invariant to 
the choice of the inverse function as shown in Lemma 3.1 of \cite{LiangKS2024EFI}. In other words, for the model (\ref{dataGeq}), 
the EFD of $\bZ_n$ is asymptotically invariant to the inverse function we learned given its consistency.

Let $\Theta:=\{\btheta \in \mathbb{R}^p: \btheta=G^*(\bY_n,\bX_n,\bz), \bz\in \mathcal{Z}_n\}$ denote the parameter space of the target model, which represents the set of all possible values of $\btheta$ that $G^*(\cdot)$ takes when $\bz$ runs over $\mathcal{Z}_n$.
Then, for any function $b(\btheta)$ of 
interest, its EFD $\mu_n^*(\cdot|\bY_n,\bX_n)$ associated with 
 $G^*(\cdot)$ is given by 
\begin{equation} \label{EFDeq}
\begin{split}
\mu_n^*(B|\bY_n,\bX_n) &=\int_{\mathcal{Z}_n(B)} d P_n^*(\bz|\bY_n,\bX_n,\bw_n^*),   \quad \mbox{for any measurable set $B \subset \Theta$},
\end{split}
\end{equation}
where $\mathcal{Z}_n(B)=\{\bz\in \mathcal{Z}_n: b(G^*(\bY_n,\bX_n,\bz)) \in B\}$. The EFD provides an uncertainty measure for $b(\btheta)$. 
Practically, 
the EFD of $b(\btheta)$ can be constructed based on the samples
$\{b(\bar{\btheta}_1), 
b(\bar{\btheta}_2), \ldots, b(\bar{\btheta}_M)\}$, where 
$\{\bar{\btheta}_1, \bar{\btheta}_2, \ldots, \bar{\btheta}_M\}$ denotes  
the fiducial $\bar{\btheta}$-samples collected at step (iv) of Algorithm \ref{EFIalgorithm}. 

Finally, we note that, \textcolor{black}{as discussed in \cite{LiangKS2024EFI}, the invariance property of $\mathcal{Z}_n$ is not crucial to the validity of EFI, although it does enhance the robustness of the inference.}
Additionally, for a neural network model, its parameters are only unique up to certain loss-invariant transformations, such as reordering hidden neurons within the same hidden layer or simultaneously altering the sign or scale of certain connection weights, see \cite{SunSLiang2021} for discussions. 
Therefore, in EFI, the consistency of $\bw_n^*$ refers to its consistency with respect to one of the equivalent 
solutions to (\ref{identityeq}), 
while mathematically $\bw_n^*$
can still be treated as unique. Refer to Section \ref{Sect:outline} (of the supplement) for more discussions on this issue.


\section{EFI for Large Models} \label{sect:EFI-large}

In this section, we first establish the consistence of the inverse function/mapping learned in EFI for large models, and then discuss its application for uncertainty quantification of deep neural networks. 

\subsection{Consistency of Inverse Mapping Learned in EFI for Large Models}

It is important to note that the sparse deep learning theory of  \cite{SunSLiang2021} is developed under the general constraint $dim(\bw_n)=O(n^{1-\delta})$ for some $0<\delta<1$, 
which restricts the dimension of the output layer of the DNN model to be fixed or grows very slowly with the sample size $n$.
Therefore,  under its current theoretical framework, 
EFI can only be applied to the models for which the dimension is fixed or increases 
very slowly with $n$. 
 
To extend EFI to large models, where 
the dimension of $\btheta$ can grow with $n$ at a rate of $O(n^{\zeta})$, 
particularly for $1/2\leq \zeta<1$, we provide a new proof for 
the consistency of  $G^*(\bY_n,\bX_n,\bZ_n)$ based on the theory 
of stochastic deep learning \citep{LiangSLiang2022}. 
Specifically, we establish the following theorem, where the output layer width of the DNN in the EFI network is set to match the dimension of $\btheta$. The proof is lengthy and provided in the supplement.

\begin{theorem} \label{thm:largemodel} Suppose Assumptions \ref{ass1}-\ref{ass6} hold (see the supplement), $\epsilon$ is sufficiently small, and 
\begin{equation} \label{structureq}
     \sum_{l=1}^H d_l \prec n,
 \end{equation}
 where $d_l$ denotes the width of layer $l$, $d_H=dim(\btheta)$, and $H$ denotes the depth of the DNN in the EFI network. Then $G^*(\bY_n,\bX_n,\bZ_n)= \frac{1}{n} \sum_{i=1}^n \hat{g}(y_i,\bx_i,z_i,\bw_n^*)$ constitutes a consistent estimator of the inverse function. 
\end{theorem} 

As implied by (\ref{structureq}), we have 
$d_l \prec n$ holds for each layer $l=1,2,\ldots,H$. We call such a neural network a narrow DNN. 
For narrow DNNs, by the existing theory, see e.g., \cite{Kidger2020UniversalAW}, \cite{Park2020MinimumWF}, and \cite{Kim2023MinimumWF}, 
the universal approximation can be achieved with 
a minimum hidden layer width of  
 $\max\{d_0+1, d_H\}$, where $d_0$ and $d_H$ represent the widths of 
the input and output layers, respectively.
Hence, (\ref{structureq}) implies that EFI can be applied to 
statistical inference for a large model of dimension 
\[
dim(\btheta)=d_H =O(n^{\zeta}), \quad 0 \leq \zeta<1,
\]
under the narrow DNN setting
with the depth $H=O(n^{\beta})$ for some $0<\beta<1-\zeta$. 
Here, Without loss of generality, we assume $d_0 \preceq d_H$.
For such a DNN, the total dimension of $\bw_n$: 
\[
dim(\bw_n)=\sum_{i=1}^H d_i (d_{i-1}+1) =O(n^{2\zeta+\beta}),
\]
can be much greater than $n$, where `1' represents the bias parameter of each neuron at the hidden and output layers. 
Specifically, we can have $dim(\bw_n) \succ n$ with appropriate choices of $\zeta$ and $\beta$.
However, leveraging the asymptotic equivalence between the DNN and an auxiliary stochastic neural network (StoNet) \citep{LiangSLiang2022}, we can still prove that the resulting 
estimator of $\btheta$ is consistent, see the supplement for the detail. 

Regarding this extension of the EFI method for statistical inference of large models, we have an additional remark: 

\begin{remark} 
In this paper, we impose a mixture Gaussian prior on $\bw_n$ to ensure the consistency of $\bw_n^*$ and, consequently, 
the consistency of the inverse mapping $G^*(\bY_n,\bX_n,\bZ_n)$. 
However, this Bayesian treatment of $\bw_n$ is not strictly necessary, although it introduces sparsity that improves the efficiency of EFI.
For the narrow DNN, the consistency of the $\bw_n$ estimator can also be established under the frequentist framework by leveraging the asymptotic equivalence between the DNN and the auxiliary StoNet, using the same technique introduced in the supplement (see Section \ref{sect:consistency}). In this narrow and deep setting, each of the regressions formed by the StoNet is low-dimensional (with $d_l \prec n$), making the Bayesian treatment of $\bw_n$ unnecessary while still achieving a consistent estimator of $\bw_n$.
\end{remark}

\subsection{Double-NN Method}

Suppose a DNN  is used for modeling the data, i.e., approximating the function $f(\cdot)$ in (\ref{modeleq}). 
By \cite{SunSLiang2021} and \cite{Farrell2021DeepNN}, a DNN of 
size $O(n^{\zeta})$ for some $0<\zeta<1$ has been large enough for approximating  many classes 
of functions. Therefore, EFI can be used for making inference for such a 
DNN model. 
In this case, EFI involves two neural networks, one is for modeling the data, which is called the `data modeling network' and parameterized by $\btheta$; 
and the other one is for approximating the inverse function, which is called the `inverse mapping network' and parameterized by $\bw_n$. 
Therefore, the proposed method is coined as `double-NN'. 
\textcolor{black}{
Note that during the EFI training process, only the parameters  $\bw_n$ of the inverse mapping network are updated in equation (\ref{sgd_update}) of Algorithm \ref{EFIalgorithm}. The parameters of the data modeling network are subsequently updated in response to the adjustment of $\bw_n$, based on the formula given in  (\ref{thetabareq}).} 

In our theoretical study for the double-NN method, we actually assume that the true data-generating model $Y=f(\bX,Z,\btheta)$ is a neural network, thereby \textcolor{black}{omitting} the approximation error of the data modeling network, based on its universal approximation capability. \textcolor{black}{
In practice, we have observed that the double-NN method is robust to this  approximation error. Specifically, even when the true model is not a neural network, EFI can still recover the true random errors with high accuracy and achieve the zero-energy solution as $n\to \infty$ and $\epsilon\to 0$. A further theoretical exploration of this phenomenon would be of interest.}

As mentioned previously, for 
 a neural network model, its parameters are only unique up to certain loss-invariant transformations. 
As the training sample size $n$ becomes large, we expect that the optimizers 
$\hat{\btheta}:=\arg\max_{\btheta} \pi_{\epsilon}(\bZ_n|\bX_n,\bY_n,\btheta)$ are all equivalent.
Thus, in this paper, the consistency of $\hat{\btheta}$ refers to its consistency with respect to one of the equivalent global optimizers, 
while mathematically $\hat{\btheta}$
can still be treated as unique. A similar issue occurs to the parameters of the inverse mapping network, as discussed in Section \ref{Sect:outline} of the supplement.

\section{An Illustrative Example for EFI} \label{illusexample}

To illustrate how EFI works for statistical inference problems, we consider 
a  linear regression example: 
\begin{equation} \label{causalmodel1}
y_i=\tau T_i +\mu+\bx_i^{\top} \bbeta+\sigma z_i, \quad i=1,2,\ldots,n,
\end{equation}
where $T_i \in \{0,1\}$ is a binary variable indicating the treatment assignment,  
$\tau$ is the treatment effect, $\bx_i \in \mathbb{R}^d$ are confounders/covariates, $z_i\sim N(0,1)$ is the standardized random noise, and $\bbeta \in \mathbb{R}^d$ and $\sigma\in \mathbb{R}_+$ are unknown parameters. 
For this example, $\tau$ represents the ATE as well as the CATE, 
due to its independence of the covariates $\bx$. 
In the simulation study, we set $\tau=1$, $\mu=1$, $d=4$, and $\bbeta=(-1,1,-1,1)^{\top}$; 
 generate $\bx_i \sim N(0,I_d)$; and   
 generate the treatment variable via a logistic regression:
\begin{equation} \label{treatmodel1}
P(T_i=1)=\frac{1}{1+\exp\{-\nu-\bxi^{\top} \bx_i \}},
\end{equation}
where $\nu=1$ and $\bxi=(-1,1,-1,1)^{\top}$.  
We consider three different cases with the sample size 
$n=250$, 500 and 1000, respectively. For each case, we  generate 100 datasets. 

Statistical inference for the parameters in the model (\ref{causalmodel1}) can be made  with EFI under its standard framework.  Let $\btheta=(\tau,\mu,\bbeta^{\top},\log \sigma)^{\top}$ be the parameter vector.  EFI approximates the inverse function $\btheta=g(y,T,\bx,z)$ by a DNN, for which $(y,T,\bx,z)$ serves as input variables and $\btheta$ as output variables. The results are summarized in Table \ref{causal:ATE}. 

For comparison, a variety of methods, including Unadj \citep{imbensrubin2015}, inverse probability weighting (IPW) \citep{Rosenbaum1987ModelBasedDA}, double-robust (DR) \citep{rrz94, Bang2005DoublyRE}, and BART \citep{Hill2011bart}, have been applied to this example. These methods fall into distinct categories. 
The Unadj is straightforward, estimating the ATE by calculating the difference between the treatment and control groups, i.e., $\hat{\tau}=\frac{1}{n_t} \sum_{i=1}^{n_t} Y_i(1)- \frac{1}{n_c} \sum_{i=1}^{n_c} Y_i(0)$,
where the effect of confounders is not adjusted. 
Both IPW and DR are widely used ATE estimation methods, which adjust the effect of confounders based on propensity scores. They both are implemented using the R package {\it drgee} \citep{drimplement}. 
The BART employs Bayesian additive regression trees to learn the outcome function,  
which naturally accommodates  heterogeneous treatment effects as well as 
nonlinearity of the outcome function. \textcolor{black}{It is implemented using the R package {\it bartcause} \citep{Dorie2020CausalIU}. }

 \begin{table}[htbp]
\caption{Comparison of EFI with various ATE estimation methods, where “coverage” refers to the averaged coverage rate of  $\tau$,  “length” refers to the averaged width of confidence intervals, and the number in the parentheses refers to the standard deviation of the averaged width. The averages and standard deviations  were calculated based on 100 datasets.}
\label{causal:ATE}
\vspace{-0.15in}
\begin{center}
\begin{tabular}{crcccrcccc} \toprule
   &\multicolumn{2}{c}{$n=250$} & &  \multicolumn{2}{c}{$n=500$} & &  \multicolumn{2}{c}{$n=1000$}\\  
  \cline{2-3} \cline{5-6}  \cline{8-9} 
  Method  & coverage & length & &  coverage & length & & coverage & length\\ 
  \midrule
 Unadj  & 0.95 & 1.161(0.066) & &  0.93 & 0.822(0.032)  & & 0.97 & 0.424(0.017)\\ 
 BART   & 0.99 & 0.857(0.070) & &   0.98  & 0.611(0.047) & & 0.96 & 0.428(0.024)\\ 
IPW  & 0.90 &  0.710(0.157) & &  0.92 & 0.560(0.141) & & 0.92 & 0.417(0.101)\\ 
DR  & 0.96 & 0.652(0.058) & &   0.93 & 0.465(0.033) & & 0.94 & 0.331(0.017)\\ 
EFI  & 0.95& 0.647(0.033) & &  0.95 & 0.438(0.021) & & 0.95 & 0.338(0.012)\\ 
\bottomrule
\end{tabular}
\end{center}
\end{table}

The comparison indicates that EFI performs very well for this standard ATE estimation problem. Specifically, EFI generates confidence intervals of nearly the same length as DR, but with more accurate coverage rates. 
This is remarkable, as DR has often been considered as the golden standard for ATE estimation and is consistent if either the outcome or propensity score models is correctly specified,  and locally efficient if both are correctly specified. 
Furthermore, EFI produces much shorter confidence intervals compared to Unadj, IPW, and BART, while maintaining more accurate coverage rates. 


We attribute the superior performance of EFI on this example 
to its fidelity in parameter estimation, an attractive property of EFI as discussed in \cite{LiangKS2024EFI}. As implied by (\ref{eqA13}), EFI 
essentially estimates $\btheta$ by maximizing the predictive likelihood function 
$\pi_{\epsilon}(\bZ_n|\bX_n,\bY_n,\btheta)\propto \pi_0^{\otimes n}(\bZ_n) e^{- 
U_n(\bY_n,\bX_n,\bZ_n, \bw)/\epsilon}$, which balances the likelihood of $\bZ_n$
and the model fitting errors coded in $U_n(\cdot)$. 
 In contrast, the maximum likelihood estimation (MLE) method sets $\hat{\btheta}_{MLE}= \arg\max_{\btheta}\pi_0^{\otimes n}(\bZ_n)$, where $\bZ_n$ is expressed as a  function of $(\bY_n,\bX_n,\btheta)$. In general, MLE is \textcolor{black}{inclined} to be influenced by the outliers and deviations of covariates especially when the sample size is not sufficiently large. 
 It is important to note that the MLE serves 
 as the core for all the IPW, DR and BART methods in estimating the outcome and propensity score models. For this reason, various adjustments for confounding and    
  heterogeneous treatment effects have been developed in the literature. 

Compared to the existing causal inference methods, EFI works as a solver for the data-generating equation (as $\epsilon \downarrow 0$), providing a coherent way to address the confounding and heterogeneous treatment 
effects and resulting in faithful estimates for the model parameters and their uncertainty as well.  
 This example illustrates the performance of EFI in ATE estimation when confounders are present, while the examples in the next section showcase the performance of EFI in 
dealing with heterogeneous treatment effects via DNN modeling. Extensive 
comparisons with BART and other nonparametric modeling methods are also presented. 

In this example, we omit the estimation of the propensity score model. As discussed in Section \ref{discsection}, the proposed method can be extended by including an additional DNN to approximate the propensity score, enabling the use of inverse probability weighting for ATE estimation. However, the ATE estimation is not the focus of this work.

\section{Causal Inference for Individual Treatment Effects} 

This section demonstrates how EFI can be used to perform statistical inference of 
the predictive ITE for the data-generating model (\ref{dataGeq}). 
Let $\btheta_c$ denote the vector of parameters for modeling the function $c(\bx)$, 
let $\btheta_{\tau}$ denote the vector of parameters for modeling the function 
$\tau(\bx)$, and let $\btheta=\{\btheta_c,\btheta_{\tau},\log(\sigma)\}$ denote the whole set of parameters for the model (\ref{dataGeq}). 
We model the inverse function $\btheta=g(y,T,\bx,z)$ 
by a DNN. Also, we can model each of the functions  
$c(\bx)$ and $\tau(\bx)$ by a DNN if their functional forms 
are unknown. For convenience, we refer to the DNN for modeling $c(\bx)$ as `$c$-network' 
and that for modeling $\tau(\bx)$ as `$\tau$-network', and  $\btheta_c$ and $\btheta_{\tau}$ represent their weights, respectively.  
As mentioned previously, we can restrict the sizes of the $c$-network and $\tau$-network
to the order of $O(n^{\tilde{\zeta}})$ for some $0<\tilde{\zeta}<1$.

Note that in solving the data generating equations (\ref{dataGeq}), the proposed method involves two types of neural networks: 
one for modeling  causal effects and the other for 
approximating the inverse  function $\btheta=g(y,T,\bx,z)$.
While we still refer to the proposed method as `Double-NN',  it actually 
involves three DNNs. 


\subsection{ITE prediction intervals} \label{ITEPsect}


Assume the training set consists of $n_{train}$ subjects,  and the test set  
consists of $n_{test}$ subjects. The subjects in the test set can be grouped into 
three categories: (i) $\{(\bx_i,0,Y_i(1),Y_i^{obs}(0))$, $i\in \mathcal{I}_c\}$, where the responses under the control are observed; 
(ii) $\{(\bx_i,1,Y_i^{obs}(1),Y_i(0)): i \in \mathcal{I}_t\}$, where 
the responses under the treatment are observed; and (iii) 
$\{(\bx_i,T_i,Y_i(1),Y_i(0)): i \in \mathcal{I}_m\}$, where only covariates are observed. 
Here, we use $\mathcal{I}_c$, $\mathcal{I}_t$, and $\mathcal{I}_m$ to denote the index sets of
the subjects in the respective categories and, therefore, $\mathcal{I}_c\cup \mathcal{I}_t\cup \mathcal{I}_m=\{1,\dots,n_{test}\}$.
For the ITE of each subject in the test set, 
we can construct the prediction interval with a desired confidence level of $1-\alpha$ in the following procedure: 

\begin{itemize}
    \item[(i)] 
    {\it For subject $i\in \mathcal{I}_c$}: 
    At each iteration $k$ of Algorithm \ref{EFIalgorithm}, calculate  the prediction  $\hat{Y}_i^{(k)}(1)=\hat{c}^{(k)}(\bx_i)+ \hat{\tau}^{(k)}(\bx_i)+\hat{\sigma}^{(k)} Z_{new}^{(k,1)}$, where $ Z_{new}^{(k,1)}\sim N(0,1)$. Let $c_l(\bx_i,1)$ and $c_u(\bx_i,1)$ denote, respectively, the  
 $\frac{\alpha}{2}$- and $(1-\frac{\alpha}{2})$-quantiles of   $\{\hat{Y}_i^{(k)}(1): k=\mK+1, \mK+2,\ldots, \mK+M\}$
 collected over iterations. 
 Since $Y_i^{obs}(0)$ is observed,  $(c_l(\bx_i,1)-Y_i^{obs}(0),c_u(\bx_i,1)-Y_i^{obs}(0))$ forms a $(1-\alpha)$-prediction interval for the ITE $Y_i(1)-Y_i^{obs}(0)$.

    \item[(ii)] 
   {\it For subject $i\in \mathcal{I}_t$}: 
    At each iteration $k$ of Algorithm \ref{EFIalgorithm}, calculate  the prediction   $\hat{Y}_i^{(k)}(0)=\hat{c}^{(k)}(\bx_i)+\hat{\sigma}^{(k)} Z_{new}^{(k,2)}$, where $ Z_{new}^{(k,2)}\sim N(0,1)$. Let $c_l(\bx_i,0)$ and $c_u(\bx_i,0)$ denote, respectively, the  
 $\frac{\alpha}{2}$- and $(1-\frac{\alpha}{2})$-quantiles of   $\{\hat{Y}_i^{(k)}(0): k=\mK+1, \mK+2,\ldots, \mK+M\}$
 collected over iterations. Since $Y_i^{obs}(1)$ is observed,  $(Y_i^{obs}(1)-c_u(\bx_i,0), Y_i^{obs}(1)-c_l(\bx_i,1))$ forms a $(1-\alpha)$-prediction interval for the ITE $Y_i^{obs}(1)-Y_i(0)$. 
    

    \item[(iii)] 
  {\it  For subject $i\in \mathcal{I}_m$}: 
At each iteration $k$ of Algorithm \ref{EFIalgorithm}, calculate  the prediction   $\hat{Y}_i^{(k)}(1)-\hat{Y}_i^{(k)}(0)=\hat{\tau}^{(k)}(\bx_i)+\sqrt{2} \hat{\sigma}^{(k)} Z_{new}^{(k,3)}$, where $ Z_{new}^{(k,3)}\sim N(0,1)$. Let $c_l(\bx_i)$ and $c_u(\bx_i)$ denote, respectively, the  
 $\frac{\alpha}{2}$- and $(1-\frac{\alpha}{2})$-quantiles of   $\{
 \hat{Y}_i^{(k)}(1)-\hat{Y}_i^{(k)}(0): k=\mK+1, \mK+2,\ldots, \mK+M\}$
 collected over iterations. 
 Then $(c_l(\bx_i), c_u(\bx_i))$ forms a $(1-\alpha)$-prediction interval for the ITE $Y_i(1)-Y_i(0)$. 
    
\end{itemize}

\subsection{Simulation Study}



\paragraph{Example 1} Consider the data-generating equation  
\begin{equation} \label{dataGex1}
y_i=\mu+\bx_i^{\top} \bbeta+(\eta_0 +\eta(\bx_i)) T_i +\sigma z_i, \quad i=1,2,\ldots,n,
\end{equation}
where $\bx_i=(x_{i,1},x_{i,2})^{\top}$ with each element drawn independently from $Unif(0,1)$, 
$\mu=1$, $\bbeta=(1,1)^{\top}$, $\eta_0=1$, $\sigma=1$,  $z_i \sim N(0,1)$, and 
$\eta(\bx_i)=s(x_{i1})s(x_{i2})-E(s(x_{i1})s(x_{i2}))$. 
As in \cite{lei2021ite}, we set $s(a)=\frac{2}{1+exp(-12(a-0.5))}$, and generate 
 the treatment variable $T_i$ according to the  propensity score model:
\begin{equation} \label{propensityeq}
e(\bx_i)=\frac{1}{4}(1+\beta_{2,4}(x_{i,1})),
\end{equation}
where $\beta_{2,4}$ is the CDF of the beta distribution with parameters (2,4), ensuring $e(\bx_i) \in [0.25,0.5]$ and thereby sufficient overlap between the treatment and control groups.
In terms of equation (\ref{dataGeq}), we have $c(\bx_i)=\mu+\bx_i^{\top}\bbeta$ and $\tau(\bx_i)=\eta_0 +\eta(\bx_i)$. 
We generated 20 datasets from  the model (\ref{dataGex1}) independently, each consisting of $n_{train}=500$ training samples and $n_{test}=1000$ test samples.

\begin{table}[!h]
\caption{
Comparison of  Double-NN and CQR for inference of the predictive 
ITE for Example (\ref{dataGex1}), where the coverage and length of the prediction intervals were calculated by averaging over 20 datasets with the standard deviation given in the parentheses. 
}
\vspace{-0.15in}
\label{case1compar}
\begin{center}
\begin{adjustbox}{width=1.0\textwidth}
\begin{tabular}{ccccccccc} \toprule
     & \multicolumn{2}{c}{Case $\mathcal{I}_c$} & & \multicolumn{2}{c}{Case $\mathcal{I}_t$} 
     &  & \multicolumn{2}{c}{Case $\mathcal{I}_m$}\\ \cline{2-3} \cline{5-6} \cline{8-9}
 Method &  Coverage & Length & & Coverage & Length & Method & Coverage & Length \\ \midrule
  \multirow{2}{*}{Double-NN}  & 0.9549 & 4.2004 & &  0.9581 & 4.1812 & 
  \multirow{2}{*}{Double-NN}  &0.9583 & 5.6056\\
  & (0.0095) & (0.1567) & &  (0.0098)  & (0.1541) &   & (0.0103) & (0.2207)\\
  \multirow{2}{*}{CQR-BART}     &  0.9472 & 4.2702 & & 0.9533 & 4.4024 & \multirow{2}{*}{CQR(inexact)} & 0.9530 & 6.3244\\
 &  (0.0342) & (0.5225) & & (0.0341) & (0.8972) & & (0.0198) & (0.5426)\\
 \multirow{2}{*}{CQR-Boosting} & 0.9556 & 5.5199 & & 0.9548 & 4.4493& 
 \multirow{2}{*}{CQR(exact)}  & 1.0000 & 13.4005\\
  & (0.0294) & (0.5866) & & (0.0259) & (0.5097) &  & (0.0002) & (2.4936) \\
 \multirow{2}{*}{CQR-RF} &  0.9529 & 5.4609 & & 0.9652 & 4.6428 & \multirow{2}{*}{CQR(naive)} & 0.9998  & 12.8861\\ 
  &   (0.0233) & (0.5172) & & (0.0171) & (0.5408) &  & (0.0004) & (1.5275)\\ 
 \multirow{2}{*}{CQR-NN} &  0.9570 & 6.4072
 & & 0.9755 & 5.8125 & \\ 
  &  (0.0195) & (0.8087) & & (0.0199) &  (1.4332) &  & \\ 

  \bottomrule
\end{tabular}
\end{adjustbox}
\end{center}
\end{table}

For this example, we assume the functional form of $c(\bx)$ is known and model  
$\tau(\bx)$ by a DNN. The DNN has two hidden layers,  each consisting of 10 hidden neurons. The number of parameters of the DNN is $|\btheta_{\tau}|=151$, and the total dimension of $\btheta=(\mu, \eta_0,\bbeta,\btheta_{\tau}^{\top},\log(\sigma))^{\top}$ is 156 $(\approx n_{train}^{0.81})$, which falls into the class of large models. 


Refer to Section \ref{parasetting} of the supplement for parameter settings 
for the Double-NN method.  
For comparison, the conformal quantile 
regression (CQR) method \citep{Romano2019ConformalizedQR,lei2021ite} was applied to this example, where the outcome function was approximated using different machine learning methods, including 
BART \citep{Chipman2010BARTBA}, Boosting \citep{Schapire1990TheSO,Breiman1998ArcingC}, and random forest (RF) \citep{breiman2001random}, and neural network (NN).
Refer to Section \ref{CQR} of the supplement for a brief description of the CQR method. 
\textcolor{black}{For CQR-NN, we used a neural network of structure $(p+1)$-10-10-2 to model the outcome quantiles, where the extra input variable is for treatment and the two output neurons are for $(\alpha/2,1-\alpha/2)$-quantiles of the outcome
\citep{Romano2019ConformalizedQR}. Additionally,
we used a neural network of structure  $p$-10-10-1 to model the propensity score in order to compute weighted CQR as in \cite{lei2021ite}.}

\textcolor{black}{The other CQR methods were implemented using the R package {\it cfcausal} \citep{lei2021ite}.}
\textcolor{black}{For the case  $\mathcal{I}_m$, we considered CQR-BART only, given its relative  
superiority over other CQR methods in 
the cases $\mathcal{I}_c$ and $\mathcal{I}_t$.}

The results were summarized in Table 
 \ref{case1compar}. The comparison shows that the Double-NN method outperforms the CQR methods in both the coverage rate and length of the prediction intervals 
 under all the three cases $\mathcal{I}_c$, $\mathcal{I}_t$, and $\mathcal{I}_m$. Specifically, the prediction intervals resulting from the Double-NN method tend to be shorter, while their coverage rates tend to be closer to the nominal level.

\begin{figure}[!h]
    \centering
    \includegraphics[width=0.85\textwidth]{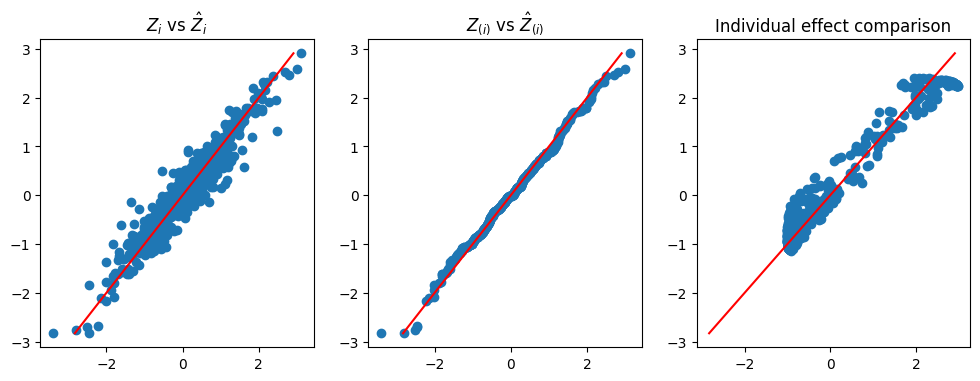}
    \caption{Demonstration of the Double-NN method for a dataset simulated from (\ref{dataGex1}): (left) scatter plot of $\hat{\bz}_i$ ($y$-axis) versus $\bz_i$ 
    ($x$-axis); (middle) Q-Q plot of $\hat{\bz}_i$  and $\bz_i$; 
    (right) scatter plot of $\tau(\bx_i)$ ($y$-axis) 
     versus $\hat{\tau}(\bx_i)$ ($x$-axis). 
     }
    \label{fig:case1}
\end{figure}

Figure \ref{fig:case1} demonstrates the rationale underlying the Double-NN method. 
The left scatter plot compares the imputed and true values of the latent 
variables for a dataset simulated from (\ref{dataGex1}), 
where the imputed values were collected at the last iteration of Algorithm \ref{EFIalgorithm}. 
The comparison reveals a close match between the imputed and true latent variable 
values, with the variability of the imputed values representing the source of uncertainty 
in the data-generating system. This variability in the latent variables can 
be propagated to $\btheta$ through the estimated inverse function $G(\cdot)$, 
leading to the uncertainty in parameters and, consequently, 
 the uncertainty in predictions.  
The middle scatter plot shows that the imputed latent variable values follows 
the standard Gaussian distribution, as expected. 
The right scatter plot compares the estimated and true values of the function 
$\tau(\bx_i)$, with the variability of the estimator representing 
its uncertainty. This plot further implies that the Double-NN method not only works for 
performing inference for the predictive ITE but also works for performing inference for CATE.

\paragraph{Example 2} 
 Consider the data-generating equation  
\begin{equation} \label{dataGex2}
y_i=c(\bx_i)+\tau(\bx_i) T_i +\sigma z_i, \quad i=1,2,\ldots,n,
\end{equation}
where $\bx_i=(x_{i,1},x_{i,2},\ldots,x_{i,5})^{\top}$ with each element drawn independently 
from $Unif(0,1)$, $\tau(\bx)$ and $T_i$ are generated as in Example 1 except that $\bx_i$ contains three extra false covariates,  
$c(\bx_i)=\frac{2x_{i,1}}{1+5x_{i,2}^2}$, $\sigma=1$, and $z_i \sim N(0,1)$. 
We simulated 20 datasets from this equation, each consisting of 
$n_{train}=1000$ training samples and $n_{test}=1000$ test samples. 

For this example, we modeled both $c(\bx)$ and $\tau(\bx)$ using DNNs. Each of the DNNs 
consists of two hidden layers, each layer consisting of 10 hidden neurons. 
\textcolor{black}{In consequence, $\btheta=(\btheta_c^{\top},\btheta_{\tau}^{\top},\log(\sigma))^{\top}$ has a total 
dimension of 363 $(\approx n_{train}^{0.85})$.} 


Similar to Example 1, we also applied the CQR methods \citep{lei2021ite} to 
this example for comparison. 
\textcolor{black}{The CQR methods were implemented as described in Example 1.}
The results were summarized in Table \ref{case2compar},
which indicates again that the Double-NN method outperforms  the CQR methods
under all the three cases $\mathcal{I}_c$, $\mathcal{I}_t$, and $\mathcal{I}_m$. 
 The prediction intervals resulting from the Double-NN method tend to be shorter, 
 while their coverage rates tend to be closer to the nominal level.

 Similar to Figure \ref{fig:case1},  Figure \ref{fig:case2} demonstrates the rationale
 underlying the Double-NN method, as well as its capability for CATE inference.
 The left plot demonstrates the variability 
 embedded in the latent variables of the data-generating system. The middle-left 
 plot shows that the imputed latent variables are distributed according to the 
 standard Gaussian distribution, as expected. The right two plots display the 
 estimates of $c(\bx_i)$ and $\tau(\bx_i)$, respectively. 
 Once again, we note that the variations of the estimates of $c(\bx_i)$ and $\tau(\bx_i)$, 
 as depicted in their respective scatter plots, reflect their uncertainty according to 
 the theory of EFI.

\begin{table}[!h]
\caption{
Comparison of Double-NN and CQR for inference of the predictive 
ITE for Example (\ref{dataGex2}), where the coverage and length of the prediction intervals were calculated by averaging over 20 datasets \textcolor{black}{with the standard deviation given in the parentheses.}}
\vspace{-0.15in}
\label{case2compar}
\begin{center}
\begin{adjustbox}{width=1.0\textwidth}
\begin{tabular}{ccccccccc} \toprule
     & \multicolumn{2}{c}{Case $\mathcal{I}_c$} & & \multicolumn{2}{c}{Case $\mathcal{I}_t$} 
     &  & \multicolumn{2}{c}{Case $\mathcal{I}_m$}\\ \cline{2-3} \cline{5-6} \cline{8-9}
 Method &  Coverage & Length & & Coverage & Length & Method & Coverage & Length \\ \midrule
 \multirow{2}{*}{Double-NN}  & 0.9519 & 4.2727 & & 0.9645 & 4.246 & \multirow{2}{*}{Double-NN}  & 0.9604 & 6.0079 \\
 & (0.0111) & (0.0101) & & (0.0069) & (0.0967) &  & (0.0946) &  (0.1363)\\
  \multirow{2}{*}{CQR-BART} & 0.9584 & 4.3586 & &  0.9545 & 4.2658 & \multirow{2}{*}{CQR(inexact)} & 0.9386 & 6.0492 \\
    & (0.0220) & (0.4392) & & (0.0230) & (0.4586) & & (0.0270) & (0.6062)\\
  \multirow{2}{*}{CQR-Boosting} & 0.9536 & 4.9942 &&  0.9572 & 4.4393 & \multirow{2}{*}{CQR(exact)} & 0.9996 & 12.1252\\
    &  (0.0175) & (0.4044) &&  (0.0194) & (0.4213) &  & (0.0007) & (1.1022) \\
  \multirow{2}{*}{CQR-RF} &  0.9563 & 5.6658 &&  0.9580 & 4.4399 & \multirow{2}{*}{CQR(naive)} &  0.9988 & 11.5566 \\
  &  (0.0198) & (0.4777) &&  (0.0232) & (0.5044) &  &  (0.0014) & (0.9309) \\
    \multirow{2}{*}{CQR-NN} &  0.9595 & 4.6748 &&  0.9452 & 3.9579 & \\
  &  (0.0165) & (0.6015) &&  (0.0185) &  (0.4301) & \\
  \bottomrule
\end{tabular}
\end{adjustbox}
\end{center}
\end{table}

\begin{figure}[!h]
    \centering
    \includegraphics[width=0.9\textwidth]{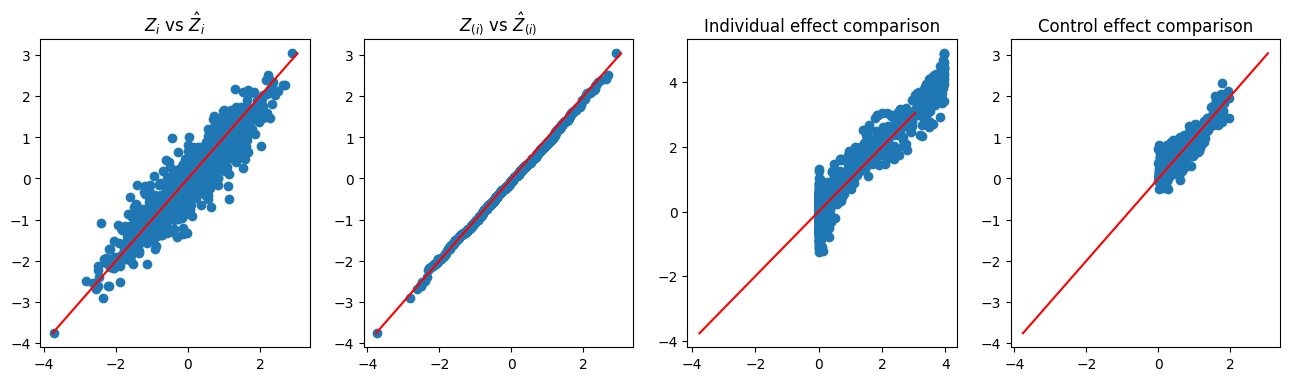}
    \caption{ Demonstration of the Double-NN method for a dataset simulated from (\ref{dataGex2}): (left) scatter plot of $\hat{\bz}_i$ ($y$-axis) versus $\bz_i$ 
    ($x$-axis); (middle-left) Q-Q plot of $\hat{\bz}_i$  and $\bz_i$; 
    (middle-right) scatter plot of $c(\bx_i)$ ($y$-axis) 
     versus $\hat{c}(\bx_i)$ ($x$-axis); 
     (right) scatter plot of $\tau(\bx_i)$ ($y$-axis) 
     versus $\hat{\tau}(\bx_i)$ ($x$-axis). } 
    \label{fig:case2}
\end{figure}

\vspace{-0.2in}
\paragraph{Precision in Estimation of Heterogeneous Effects}

As demonstrated in Figure \ref{fig:case1} and Figure \ref{fig:case2}, the Double-NN method
can also be used for inference of CATE. The performance in CATE estimation 
is often measured using the expected Precision 
in Estimation of Heterogeneous Effects (PEHE), which is defined as:
\[
\epsilon_{PEHE}=\int_{\mathcal{X}}(\hat{\tau}(\bx)-\tau(\bx))^2 dF(\bx),
\]
where $F(\bx)$ denotes the distribution function of the covariates $\bX$. 
As we can see, $\epsilon_{PEHE}$ summarizes the precision of the CATE over the entire 
sample space $\mX$ \citep{Hill2011bart,shalit2017pehe,caron2022itereview}.  
In practice, since we only observe the treatment effect on the treatment group, 
the target of interest is generally only for the treatment group, i.e  $\epsilon_{PEHE}^{(T)}=\int_{\mathcal{X}}(\hat{\tau}(\bx)- \tau(\bx))^2 dF_T(\bx)$, where $F_T(\bx)$ denotes the distribution function of the covariates in the treatment group. 
We estimated $\epsilon_{PEHE}^{(T)}$ by $\hat{\epsilon}_{PEHE}^{(T)}=\frac{1}{n_t}\sum_{i\in I_{t}}(\hat{\tau}(\bx_i)-\tau(\bx_i))^2$. 
For the Double-NN method, we set $\hat{\tau}(\bx_i)=\frac{1}{M}\sum_{k=\mK+1}^{\mK+M}\hat{\tau}^{(k)}(\bx_i)$, where $M$ denotes the number of estimates of $\tau(\bx_i)$ 
collected in a run of Algorithm \ref{EFIalgorithm}. 

For comparison, the existing CATE estimation methods, including single-learner (S-learner), 
two-learner (T-learner), and X-learner \citep{Knzel2019MetalearnersFE}, 
have been applied to the datasets generated above, where the RF and BART are used 
as the base learners. In the S-learner, a single outcome function is estimated using a base learner 
with all available covariates,
where the treatment indicator is treated as a covariate, and then estimate CATE by 
$\hat{\tau}_S=\hat{\mu}(\bx,1)-\hat{\mu}(\bx,0)$,
where $\hat{\mu}(\bx,t)$ denotes the outcome function estimator. 
The T-learner estimates the outcome functions using a base learner separately for the units under the control and those under the treatment, and then estimate CATE by 
$\hat{\tau}_T(\bx)=\hat{\mu}_1(\bx)-\hat{\mu}_0(\bx)$,
where $\hat{\mu}_t(\bx)$ denote the outcome function estimator for the assignment group 
$t \in \{0,1\}$. The X-learner builds on the T-learner; it uses the observed outcomes 
to estimate the unobserved ITEs, and then estimate the CATE in another step as if 
the ITEs were observed. \textcolor{black}{Refer to \cite{Knzel2019MetalearnersFE} and \cite{caron2022itereview} for the detail. 
We implemented the S-learner, T-learner, and X-leaner using the package downloaded at \url{https://github.com/albicaron/EstITE}.}  

Table \ref{tab:PEHE} compares the values of $\hat{\epsilon}_{PEHE}^{(T)}$ 
resulting from the Double-NN, S-learners, T-learners, and X-learners. 
 for the models (\ref{dataGex1}) and (\ref{dataGex2}).
The comparison shows that the Double-NN method outperforms the existing ones in achieving 
consistent CATE estimates over different covariate values.  This is remarkable!
As explained in Section \ref{illusexample}, we would attribute this performance of the Double-NN method to its fidelity in parameter estimation \citep{LiangKS2024EFI}. 
Compared to the MLE method, which serves as the prototype for the base learners, 
the Double-NN method is forced to be more robust to covariates due to 
added penalty term $U_n(\bY_n,\bX_n,\bZ_n, \bw_n)/\epsilon$.

\begin{table}[!h]
\caption{Comparison of Double-NN and other methods in $\epsilon_{PEHE}^{(T)}$, 
where each of the mean and standard deviations  was calculated based on 20 
datasets generated from  (\ref{dataGex1}) or   (\ref{dataGex2}).}
\label{tab:PEHE}
\vspace{-0.15in}
\begin{center}
\begin{tabular}{cccccc} \toprule
   &\multicolumn{2}{c}{Model (\ref{dataGex1})} &  &\multicolumn{2}{c}{Model (\ref{dataGex2})} \\ \cline{2-3}  \cline{5-6} 
  Method  & Training & Test & & Training & Test\\ 
  \midrule
 S-RF  &     0.3769 $\pm$ 0.0170  & 0.3660 $\pm$ 0.0188    &&  0.3722 $\pm$ 0.0074    & 0.3377 $\pm$ 0.0100\\ 
 S-BART  & 0.4233 $\pm$ 0.0156 &   0.4344 $\pm$ 0.0149   &&   0.3371 $\pm$ 0.0099 &   0.3418 $\pm$ 0.0102 \\ 
\midrule 
 T-RF  &  0.4545 $\pm$ 0.0114  &  0.4198 $\pm$ 0.0118     &&  0.4095 $\pm$ 0.0064    &  0.3488 $\pm$ 0.0084\\ 
 T-BART  & 0.4190 $\pm$ 0.0139   &  0.4236 $\pm$ 0.0127    &&   0.4308 $\pm$ 0.0092  &    0.4298 $\pm$ 0.0093 \\ 
\midrule 
 X-RF  &    0.3416 $\pm$ 0.0153   &  0.3451 $\pm$ 0.0162      &&  0.2761 $\pm$ 0.0106 & 0.2789 $\pm$ 0.0106 \\ 
 X-BART  &   0.3863 $\pm$ 0.0137  &   0.3972 $\pm$ 0.0128  & &   0.3853 $\pm$ 0.0102 & 0.3862 $\pm$ 0.0097  \\ 
 \midrule 
 Double-NN  & 0.2962 $\pm$ 0.0167 & 0.3139 $\pm$ 0.0178 & & 0.3788 $\pm$  0.0105 & 0.3899 $\pm$  0.0110\\ 
\bottomrule
\end{tabular}
\end{center}
\vspace{-0.1in}
\end{table}

\subsection{Real Data Analysis}
\subsubsection{Lalonde}

The `LaLonde' data is a well-known dataset used in causal inference to evaluate the effectiveness of a job training program in improving the employment prospects of participants. We used the dataset given in the package ``twang'' \citep{twang2021} among various versions. The dataset includes earning data in 1978 on 614 individuals, with 185 receiving job training and 429 in the control group. \textcolor{black}{There are 8 covariates including  various demographic, educational, and employment-related variables.}
While the LaLonde dataset has been widely used for ATE estimation, we use it to illustrate the Double-NN method  
for constructing ITE prediction intervals. 

To evaluate the performance of different methods, we randomly split the LaLonde dataset into a training set and a test set. The training set, denoted by $\mathcal{D}_{train}$, 
consists of $n_{train}=600$ observations; while the test set, denoted by $\mathcal{D}_{test}$, consists of $n_{test}=14$ observations. We trained the Double-NN on $\mathcal{D}_{train}$ and constructed  prediction intervals for each subject in $\mathcal{D}_{test}$ with a confidence level of  $1-\alpha=0.5$. For the Double-NN, we modeled both $c(\bx)$ and $\tau(\bx)$ 
using DNNs. Each of the DNNs consists of two hidden layers, with each layer consisting of 10 hidden neurons. \textcolor{black}{In consequence, 
$\btheta=(\btheta_c^{\top},\btheta_{\tau}^{\top},\log(\sigma))^{\top}$
has a dimension of 423 $(\approx n_{train}^{0.95})$, a challenging task 
for uncertainty quantification of the model.} 

Figure \ref{fig:lalonde} displays the constructed ITE prediction intervals for the test data, comparing the proposed method to the CQR method \citep{lei2021ite}.  
The comparison shows that the prediction intervals resulting from the proposed method are shorter than those from the CQR method, while the centers of those intervals are similar. This suggests that the proposed method is able to estimate the ITEs with a higher degree of precision. 

\begin{figure}
    \centering
    \includegraphics[width=0.8\textwidth]{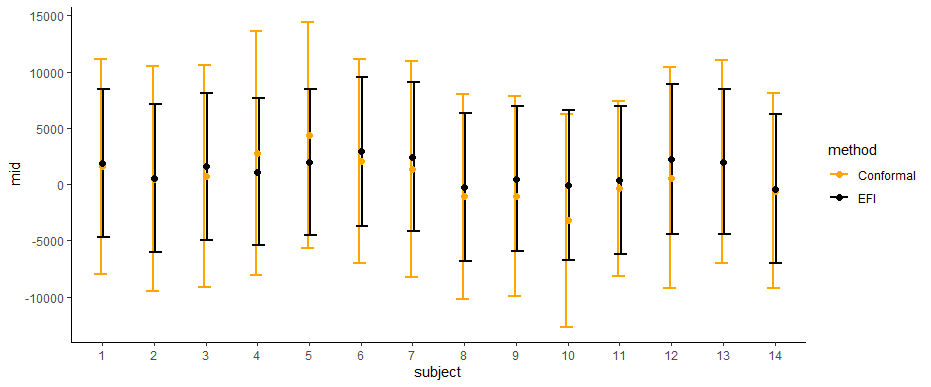}
    \caption{Comparison of prediction intervals resulting from Double-NN (labeled as EFI) and CQR (labeled as conformal) for the subjects in the test set of Lalonde.}
    \label{fig:lalonde}
\end{figure}

\subsubsection{NLSM}


This subsection conducts an analysis on
the `National Study of Learning Mindsets' (NLSM) dataset used in the 2018 Atlantic Causal Inference Conference workshop \citep{yeager2019nlsm,Carvalho2019AssessingTE}. 
\textcolor{black}{NSLM records the results of a randomized evaluation for a ``nudge-like'' intervention designed to instill students with a growth mindset.}
The dataset is available at \url{https://github.com/grf-labs/grf/tree/master/experiments/acic18}, \textcolor{black}{which includes 10,391 students from 76 schools, 
with four student-level covariates and six school-level students. 
After factoring the categorical variables, the dimension of covariates $\bx$ increases to 29.}

Due to unavailability of the true treatment effect values, we performed an exploratory analysis
as in \cite{lei2021ite}. In order to construct prediction intervals for the ITE, we split the dataset into two sets: $\mathcal{D}_{train}$ and $\mathcal{D}_{test}$. The former has a sample size of $n_{train}=5200$, and 
the latter has a sample size of $n_{test}=5191$. 
For the Double-DNN method,  we used DNNs to 
 model the functions $\tau(\bx)$ and $c(\bx)$. 
\textcolor{black}{Each DNN consists of two hidden layers, with each hidden layer 
 consisting of 10 hidden neurons. Therefore,  
the dimension of $\btheta=(\btheta_c^{\top},\btheta_{\tau}^{\top},\log(\sigma))^{\top}$ is 843 $(\approx n_{train}^{0.79})$.}

The Double-DNN was trained on $\mathcal{D}_{train}$ and the prediction intervals were constructed on $\mathcal{D}_{test}$, which corresponds to case (iii) described in Section \ref{ITEPsect}. This process was repeated 20 times. 
For comparison, the CQR method \citep{lei2021ite} was also applied to this example.

\begin{figure}[htbp]
\begin{center}
\begin{tabular}{c} 
\includegraphics[height=2.0in,width=0.55\textwidth]{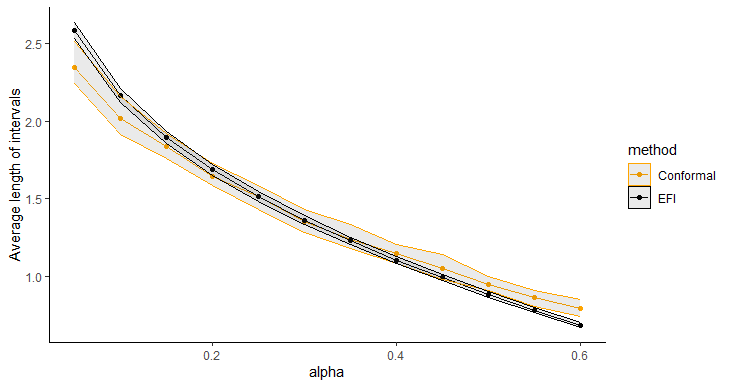}  
\end{tabular}
\end{center}
\caption{Comparison of the average length of intervals obtained by the Double-NN (labeled as EFI) and CQR (labeled as conformal) for the NLSM data.}  
\label{fig:nlsm1}
\end{figure}

Figure \ref{fig:nlsm1} 
displays the average length of prediction intervals, obtained by Double-DNN and CQR, 
as a function of $\alpha$, with
the upper and lower envelops being respectively the $95\%$ and $5\%$ 
quantiles across 20 runs. For this example, we implemented CQR using 
the ``inexact'' method, and therefore, its interval lengths tend to be short with approximate validity. 
However, as shown in Figure \ref{fig:nlsm1}, the prediction intervals resulting from the Double-NN method tend  
to be even shorter than those from CQR as $\alpha$ increases. 
Figure \ref{fig:nlsm2} (a) compares the fractions of the prediction intervals, obtained by Double-NN and CQR, that cover positive values only. While Figure \ref{fig:nlsm2} (b) compares the fractions of the prediction intervals that cover negative values only.  
In summary, the Double-NN can provide more accurate predictions for the ITE than CQR for this example. 
\textcolor{black}{Specifically, the Double-NN identified fewer subjects with significant ITEs than the CQR, as implied by Figure \ref{fig:nlsm2} (a) and (b); while each has  
a narrow prediction interval, as implied by Figure \ref{fig:nlsm1}.} 



\begin{figure}[htbp]
\begin{center}
\begin{tabular}{cc} 
    (a) & (b)  \\
        \includegraphics[height=2.0in,width=0.475\textwidth]{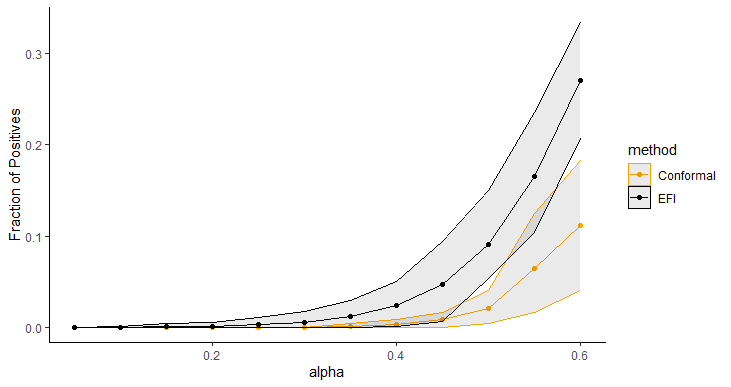}  &
         \includegraphics[height=2.0in,width=0.475\textwidth]{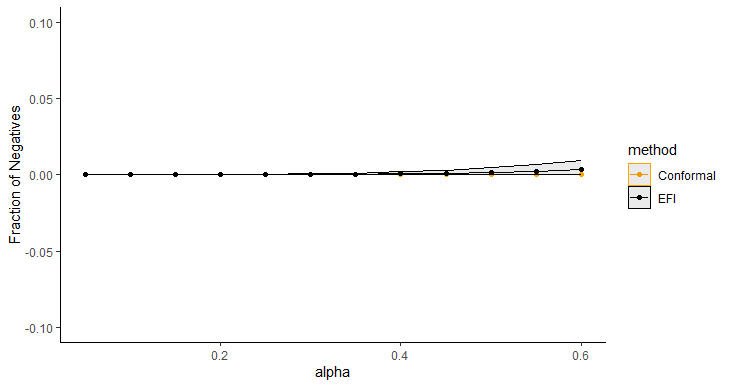} \\ 
\end{tabular}
\end{center}
\caption{Fractions of the intervals obtained by Double-NN (labeled as EFI) and CQR (labeled as conformal) with (a) positive lower bounds and (b) negative upper bounds, where the upper and lower envelops are respectively  $95\%$ and $5\%$ quantiles across 20 runs.}  
\label{fig:nlsm2}
\end{figure}

\section{Discussion}  \label{discsection}

 This paper extends EFI to statistical inference for large statistical  models and applies the proposed Double-NN method to treatment effect estimation. The numerical results demonstrate that the Double-NN method significantly outperforms the existing CQR method in ITE prediction. As mentioned in the paper, we attribute the superior performance of the Double-NN method to its fidelity in parameter estimation. Due to the universal approximation ability of deep neural networks, the Double-NN method is generally applicable for causal effect estimation. 
 
 From the perspective of statistical inference, this paper advances the theory and methodology for making inference of large statistical models, allowing the model size to increase with the sample size $n$ at a rate of $O(n^{\zeta})$ for any exponent $0 \leq \zeta<1$. 
 In particular, the Double-NN method provides a rigorous approach for quantifying the uncertainty of deep neural networks. In this paper, we have tested the performance of the Double-NN method on numerical examples with the exponent ranging $0.79 \leq \zeta \leq 0.95$, which all falls into the class of large models.

The Double-NN method can be further extended toward a general 
nonparametric approach for causal inference. Specifically, we can include an additional neural network to approximate the propensity score, enabling the outcome and propensity score functions 
to be simultaneously estimated.  
  This extension will enable the use of inverse probability weighting methods to further improve ATE estimation, especially in the scenario where the 
  covariate distributions in the treatment and control groups 
  are imbalanced \citep{shalit2017pehe,hahn2020nonpar3}. 
  From the perspective of EFI, this just corresponds to making inference 
  for a different $b(\btheta)$ function.  
  Similarly, for inference of ITE, a different $b(\btheta)$ function, including those adjusted with propensity scores, can also be used.
 The key advantage of EFI is its ability to automatically quantify the uncertainty of these functions as prescribed in (\ref{EFDeq}), even when the functions are highly complex.
  

Regarding the size of large models, our theory does not
preclude applications to large-scale DNNs with millions or even billions of parameters, as supported by the neural scaling law.
 As mentioned previously,  \cite{Hestness2017DeepLS}  investigated the relationship between the DNN model size and the dataset size: they discovered
 a sub-linear scaling law of $dim(\btheta) \prec n$ 
  across various model architectures in
  machine learning applications, including machine translation, language modeling, image processing, and speech recognition. Their findings suggest that Theorem \ref{thm:largemodel}
 remains valid for large-scale DNNs
 by choosing an appropriate growth rate for their depth.




  In practice, we often encounter small-$n$-large-$p$ problems. For such a problem, we need to deal with a model of dimension $dim(\btheta)  \succeq n$, which is often termed as 
  an over-parameterized model.  A further 
  extension of EFI for over-parameterized models is possible by imposing an appropriate sparsity constraint on $\btheta$. 
  How to make post-selection inference with EFI for the over-parameterized models will be studied in future work.

\textcolor{black}{Finally, we note that a recent work by \citet{williams2023modelfreeGFI} demonstrates how conformal prediction sets arise from a  generalized fiducial distribution.
 Given the inherent connections between GFI and EFI, we believe that the results established in \citet{williams2023modelfreeGFI} should also apply to EFI. In particular, EFI follows the same switching principle as GFI \citep{hannig2016gfi}, which infers the uncertainty of the model parameters from the distribution of unobserved random errors.
 Further research on EFI from this perspective is of great interest, as it could potentially alleviate EFI's reliance on assumptions about the underlying data distribution in prediction uncertainty quantification.
}

\section*{Availability} 

The code that implements the Double NN method can be found at \url{https://github.com/sehwankimstat/DoubleNN}.

\section*{Acknowledgments} 
  
Liang's research is supported in part by the NSF grants DMS-2015498 and DMS-2210819, and the NIH grant R01-GM152717. The authors thank the editor, associate editor, and referee for their constructive comments, which 
have led to significant improvement of this paper.


\newpage

\appendix

{\Large \bf Appendix: Supplement for ``Extended Fiducial Inference for Individual Treatment Effects via
Deep Neural Networks ''} 

\vspace{0.1in}

\setcounter{section}{0}
\renewcommand{\thesection}{$\S$\arabic{section}}
\setcounter{table}{0}
\renewcommand{\thetable}{S\arabic{table}}
\setcounter{figure}{0}
\renewcommand{\thefigure}{S\arabic{figure}}
\setcounter{equation}{0}
\renewcommand{\theequation}{S\arabic{equation}}
\setcounter{lemma}{0}
\renewcommand{\thelemma}{S\arabic{lemma}}
\setcounter{theorem}{0}
\renewcommand{\thetheorem}{S\arabic{theorem}}
\setcounter{remark}{0}
\renewcommand{\theremark}{S\arabic{remark}}

\bigskip
 This material is organized as follows. Section \ref{Proofsection} provides 
 a proof for Theorem \ref{thm:largemodel}. 
 Section \ref{CQR} provides a  brief description for the CQR method. 
 Section \ref{parasetting} provides parameter settings for the experiments 
 reported in the main text and this supplement. 

\section{Theoretical Proofs} \label{Proofsection}

To prove the validity of the proposed method, 
it is sufficient to prove that $\hat{g}(\cdot)$ constitutes a consistent estimator of $\btheta$, building on the theory developed in \cite{LiangKS2024EFI}.
To ensure the self-contained nature of this paper, we provide a concise overview of the theory presented in \cite{LiangKS2024EFI} in Section \ref{Sect:outline} of this supplement. Subsequently, our study will center on establishing the consistency of $\hat{g}(\cdot)$.

\subsection{Outline of the Proof} \label{Sect:outline}
First of all, we note that the theoretical study is conducted under the assumption 
that the EFI network has been correctly specified such that  a sparse EFI network  $\tilde{\bw}_n^*$ exists, from which the complete data $(\bX_n,\bY_n,\bZ_n^*)$ can be generated, 
where $\bZ_n^*$ represents the values of the latent variables realized in the observed samples. Specifically,  we assume $\bZ_n^* \sim \pi_{\epsilon}(\bZ|\bX_n,\bY_n,\tilde{\bw}_n^*)$ holds as $\epsilon \downarrow 0$.  
 

For the EFI network,  we define 
 \begin{equation} \label{Q1eq}
 \widehat{\mG}(\bw_n|\tilde{\bw}_n^*):=\frac{1}{n} \log\pi_{\epsilon}(\bY_n,\bZ_n^*|\bX_n,\bw_n)+\frac{1}{n} \log \pi(\bw_n). 
 \end{equation}
 Therefore, 
 \[
 \hat{\bw}_n^*:=\arg\max_{\bw_n\in \mathcal{W}_n} \widehat{\mG}(\bw_n|\tilde{\bw}_n^*),
 \]
 is the global maximizer of the posterior $\pi_{\epsilon}(\bw_n|\bX_n,\bY_n,\bZ_n^*)$.
 Also, we define 
 \begin{equation} \label{Q2eq}
 \begin{split}
 \widetilde{\mG}(\bw_n|\tilde{\bw}_n^*)& :=\frac{1}{n}\int \log \pi_{\epsilon}(\bY_n,\bZ_n^*|\bX_n,\bw_n) d\pi_{\epsilon}(\bZ_n^*|\bX_n,\bY_n,\tilde{\bw}_n^*) +\frac{1}{n} \log\pi(\bw_n). \\ 
    &=  \frac{1}{n} \Big\{ \log\pi(\bw_n|\bX_n,\bY_n) - \int \log \frac{\pi(\bZ_n^*|\bX_n,\bY_n,\tilde{\bw}_n^*)}{ \pi(\bZ_n^*|\bX_n,\bY_n,\bw_n)} d \pi(\bZ_n^*|\bX_n,\bY_n,\tilde{\bw}_n^*) \\
    & + \int \log \pi(\bZ_n^*|\bX_n,\bY_n,\tilde{\bw}_n^*) d \pi(\bZ_n^*|\bX_n,\bY_n,\tilde{\bw}_n^*)+c \Big\},\\
 \end{split}
 \end{equation}
   where $c=\log\int_{\mathcal{W}_n} \pi(\bY_n|\bX_n,\bw_n) \pi(\bw_n)d\bw_n$ is the log-normalizing constant of the posterior $\pi(\bw_n|\bX_n,\bY_n)$. Note that in the above derivation, $\bX_n$ can be ignored  for simplicity since it is constant.
   For simplicity of notation, we let 
   \[
 D_{KL}(\bw_n)= \int \log \frac{\pi(\bZ_n^*|\bX_n,\bY_n,\tilde{\bw}_n^*)}{ \pi(\bZ_n^*|\bX_n,\bY_n,\bw_n)} d \pi(\bZ_n^*|\bX_n,\bY_n,\tilde{\bw}_n^*),
 \]
 be the Kullback-Leibler divergence between $\pi(\bZ_n^*|\bX_n,\bY_n,\tilde{\bw}_n^*)$ and $\pi(\bZ_n^*|\bX_n,\bY_n,\bw_n)$. 
Regarding the EFI network, we make the following assumption: 

\begin{assumption} \label{ass1} The EFI network satisfies the conditions:
 \begin{itemize}
     \item[(i)] The parameter space $\mathcal{W}_n$ (of $\bw_n$) is convex and compact. 
     \item[(ii)] $\mathbb{E} (\log\pi_{\epsilon}(Y,Z|X,\bw_n))^2 < \infty$ for any $\bw_n\in \mathcal{W}_n$, where $(X,Y,Z)$ denotes a generic sample as those in $(\bY_n,\bZ_n,\bX_n)$.
 \end{itemize}
 \end{assumption}

Define $Q^*(\bw_n)= \mathbb{E}(\log \pi_{\epsilon}(Y,Z|X,\bw_n))+ \frac{1}{n} \log\pi(\bw_n)$.
By Assumption \ref{ass1} and the weak law of large numbers, 
\begin{equation}\label{eq:sameloss2}
    \frac{1}{n}\log\pi_{\epsilon}(\bw_n|\bX_n,\bY_n,\bZ_n)-Q^*(\bw_n)\overset{p}{\rightarrow} 0,
\end{equation}
holds uniformly over the parameter space $\mathcal{W}_n$, where $\stackrel{p}{\to}$ denotes 
convergence in probability. Further, we make the following assumption on $Q^*(\bw_n)$, which is essentially an identifiability condition of $\hat{\bw}_n^*$.

\begin{assumption}\label{ass2}
(i) $Q^*(\bw_n)$ is continuous in $\bw_n$ and is uniquely maximized
at some point $\bw_n^{\diamond}$; 
(ii) for any $\epsilon>0$, the value $sup_{\bw_n\in\mathcal{W}_n \backslash B(\epsilon)}Q^*(\bw_n)$ exists, where $B(\epsilon)=\{\bw_n:\|\bw_n-\bw_n^{\diamond}\|<\epsilon\}$, and $\delta:=Q^*(\bw_n^{\diamond})-\sup_{\bw_n\in\mathcal{W}_n \backslash B(\epsilon)}Q^*(\bw_n)>0$.
\end{assumption}
Assumption \ref{ass2} restricts the shape of $Q^*(\bw_n)$ around the global maximizer, which cannot be discontinuous or too flat.  
Given nonidentifiability of the neural network model, see e.g. \cite{SunSLiang2021}, we have implicitly assumed that each $\bw_n$  is unique up to loss-invariant transformations, e.g., reordering the hidden neurons within the same hidden layer or simultaneously altering the signs of certain weights and biases. The same assumption has often been used in theoretical studies of neural networks, see e.g. \cite{Liang2018BNN}.

On the other hand,
by Theorem 1 of \cite{liang2018imputation},  
we have 
 \begin{equation} \label{QQeq}
 \sup_{\bw_n\in \mathcal{W}_n}\left| \widehat{\mG}(\bw_n|\tilde{\bw}_n^*)-\widetilde{\mG}(\bw_n|\tilde{\bw}_n^*) \right| \stackrel{p}{\to} 0, \quad \mbox{as $n\to \infty$},
 \end{equation} 
 under some regularity conditions. Under Assumptions \ref{ass1}-\ref{ass2},
 \cite{LiangKS2024EFI} proved the following lemma: 
  
 \begin{lemma} \label{lemma:equivalent} (Lemma 4.1; \cite{LiangKS2024EFI}) Suppose Assumptions \ref{ass1}-\ref{ass2}   hold, and the joint likelihood function 
 $\pi(\bY_n,\bZ_n|\bX_n,\bw_n)$ is continuous in $\bw_n$. 
 If $\hat{\bw}_n^*$ is unique, then $\bw_n^*$ that 
 maximizes $\pi(\bw_n|\bX_n,\bY_n)$ as well as minimizes 
$D_{KL}(\bw_n)$ is unique and, subsequently, 
$\|\hat{\bw}_n^*-\bw_n^*\| \stackrel{p}{\to} 0$ holds as $n \to \infty$.
 \end{lemma}

For Lemma \ref{lemma:equivalent}, 
the uniqueness of $\hat{\bw}_n^*$, up to some loss-invariant transformations as discussed previously, can be ensured by its consistency as established 
in the followed sections of this supplement, see equation (\ref{thetaequiv2}).
The condition minimizing $D_{KL}(\bw_n)$ is generally 
implied by $\tilde{U}_n(\bY_n,\bX_n,\bZ_n,\bw_n)=0$ provided the consistency 
of $\bar{\btheta}_n^*$, while the convergence of $\bw_n^*$ to a maximizer of $\pi(\bw_n|\bX_n,\bY_n)$ is generally implied 
by the Monte Carlo nature of Algorithm \ref{EFIalgorithm}. 
Therefore, by eq. (\ref{wconvergence}) of the main text, if $\bw_n^{(k)}$ converges and 
$\tilde{U}_n(\bY_n,\bX_n,\bZ_n,\bw_n^{(k)})$ converges to 0, we would 
 have $\|\hat{\bw}_n^*-\bw_n^*\| \stackrel{p}{\to} 0$ provided that  
the prior has been appropriately chosen such that 
$\hat{\bw}_n^*$ is consistent for $\tilde{\bw}_n^*$ and $\hat{g}(y_i,x_i,z_i,\hat{\bw}_n^*)$ is consistent for  $\btheta^*$.

In summary, by Lemma \ref{lemma:equivalent}, if we can choose an appropriate prior  $\pi(\bw)$ such that $\hat{\bw}_n^*$  is consistent for $\tilde{\bw}_n^*$, then $\bw_n^*$ is also consistent for $\tilde{\bw}_n^*$
and $\hat{g}(y,\bx,z,\bw_n^*)$ is consistent for $\btheta^*$. 
Establishing the consistency of $\hat{\bw}_n^*$ will be the focus of Section \ref{sect:consistency} of this supplement. 
Note that showing the consistency of $\hat{\bw}_n^*$ is relatively simpler than directly working on $\bw_n^*$, as $\hat{\bw}_n^*$ is based on the complete data.

Then, following from the consistency of $\hat{g}(\cdot)$, we immediately have 
 \begin{equation} \label{cons*}
 \| \frac{1}{n} \sum_{i=1}^n \hat{g}(y_i,x_i, z_i,
 \bw_n^*)-\btheta^*\| \stackrel{p}{\to} 0, \quad \mbox{as $n\to \infty$}.
 \end{equation}

As a slight relaxation for the definition of $G(\cdot)$,  we can write equation (\ref{Inveq}) of the main text as 
\begin{equation} \label{ass*}
\btheta^*=\lim_{n\to \infty} G(\bY_n,\bX_n,\bZ_n),
\end{equation}
where $\bZ_n$ is assumed to be known.
By combining (\ref{cons*}) and (\ref{ass*}),  we have
\[
\left\| \frac{1}{n} \sum_{i=1}^n \hat{g}(y_i,x_i, z_i, \bw_n^*)-G(\bY_n,\bX_n,\bZ_n)\right\| \stackrel{p}{\to} 0, \quad \mbox{as $n\to \infty$},
\]
i.e., the EFI estimator $\bar{\btheta}^*:=\frac{1}{n} \sum_{i=1}^n \hat{g}(y_i,x_i, z_i, \bw_n^*)$ is consistent for the inverse mapping $G(\bY_n,\bX_n,\bZ_n)$.  
Further, by Slutsky's theorem, the uncertainty of $\bZ_n$ can be propagated to $\btheta$ via
the EFI estimator. Therefore, 
the extended fiducial density (EFD) function  of $\btheta$ can be approximated by 
\begin{equation} \label{CDestimator}
\tilde{\mu}_n(d\btheta)= \frac{1}{\mM} \sum_{k=1}^{\mM} \delta_{\bar{\btheta}^{*,k}} (d\btheta), \quad \mbox{as $\mM\to \infty$},
\end{equation}
where $\delta_a$ stands for the Dirac measure at a given point $a$, 
$\bar{\btheta}^{*,k}:= \frac{1}{n} \sum_{i=1}^n \hat{g}(x_i,y_i, z_{i}^{*,k}, \bw_n^*)$, and 
$\bZ_n^{*,k}:=(z_1^{*,k},z_2^{*,k},\ldots,z_n^{*,k})$ for $k=1,2,\ldots, \mM$ 
denotes $\mM$ random draws from the distribution 
$\pi(\bZ_n|\bX_n,\bY_n,\bw_n^*)$ under the limit setting of $\epsilon$.

\subsection{On the Consistency of $\hat{g}(\cdot)$ under the Large Model Scenario} 
\label{sect:consistency}

In this section, we first show that $\hat{\bw}_n^*$ is consistent under the framework of 
imputation-regularized optimization (IRO) algorithm \citep{liang2018imputation} by assuming that 
the true values of $\bZ_n$ are known and 
$\bw_n$ is subject to a mixture Gaussian prior. Subsequently, we show that 
$\hat{g}(\cdot)$ is consistent. 

\subsubsection{An Auxiliary Stochastic Neural Network Model}
\label{auxNN}

To show  $\hat{\bw}_n^*$ is consistent, we introduce an auxiliary stochastic neural network (StoNet) model. For each of the hidden and output neurons of the model, we introduce a random noise: 
\begin{equation} \label{eq:stochastic}
 \begin{split} 
 \tilde{u}_{l,i} &=\sum_{j=0}^{d_{l-1}} w_{n,l,i,j} u_{l-1,j}+e_{l,i}:=v_{l,i}+e_{l,i}, \\ 
 u_{l,i} & =\Psi_l(\tilde{u}_{l,i}), \\ 
\end{split}
\end{equation} 
where $l \in \{1,2,\ldots,H\}$ indexes the layers of the DNN, and $i \in \{1,2,\ldots,d_l\}$  indexes the neurons at layer $l$ of the DNN, the random noise $e_{l,i} \sim N(0, \sigma_{l}^2)$, $w_{n,l.i,j}$ denotes the weight on the connection from neuron $i$ of layer $l$ to neuron $j$ of layer $l-1$,  and $\Psi_l(\cdot)$ denotes the activation function used for layer $l$. 
Note that $\sigma_{l}^2$'s are all known and pre-specified by the user, and $l=0$ represents the input layer. 
As a consequence of introducing the random noise, we can treat $\tilde{U}_l$'s 
as latent variables, where $\tilde{U}_l=(\tilde{u}_{l,1}, \tilde{u}_{l,2}, \ldots, 
\tilde{u}_{l,d_l})^{\top}$.  

When considering a dataset of size $n$,   
we let $\tilde{u}_{l,i,(k)}$ denote the latent variable imputed for neuron $i$ 
of layer $l$ for observation $k \in \{1,2,\ldots,n\}$, 
let $\tilde{U}_{l,(k)}=(\tilde{u}_{l,1,(k)}, \tilde{u}_{l,2,(k)},\ldots, \tilde{u}_{l,d_l,(k)})^{\top}$, and 
let $\tilde{\bU}_l=\{\tilde{U}_{l,(1)}, \tilde{U}_{l,(2)},\ldots, \tilde{U}_{l,(n)} \}$ denote the latent variables at layer $l$ for all $n$ observations. 
Recall that we have defined $\bY_n=\{y_1,y_2,\ldots,y_n\}$, 
$\bX_n=\{\bx_1,\bx_2,\ldots,\bx_n\}$, and $\bZ_n=\{z_1,z_2,\ldots,z_n\}$.   
Given a set of pseudo-complete data  $(\bY_n,\tilde{\bU}_H,\ldots, \tilde{\bU}_1,\bX_n, \bZ_n)$,  we define the posterior distribution of $\bw_n$ as:  
\begin{equation} \label{decompeq}
\begin{split} 
 & \pi_{\epsilon}(\bw_n|\bY_n,\tilde{\bU}_H,\ldots, \tilde{\bU}_1,\bX_n, \bZ_n)
 \propto \pi(\bw_n) \prod_{k=1}^n p_{\epsilon}(y_k|\bx_k,z_k,\bar{\btheta})  
  \prod_{l=1}^H \prod_{k=1}^n \pi_l(\tilde{U}_{l,(k)}| \tilde{U}_{l-1,(k)},\bw_{n,l}),   
 \end{split} 
\end{equation}
where $\bw_n=\{\bw_{n,1},\bw_{n,2},\ldots,\bw_{n,H}\}$ with 
$\bw_{n,l}$ being the connection weights for layer $l$,
\[
p_{\epsilon} (y_k|\bx_k,z_k,\bar{\btheta}) \propto \exp\{-(d(y_k,x_k, z_k, \bar{\btheta})  +\eta \|\hat{\btheta}_k-\bar{\btheta} \|^2)/\epsilon \},
\]
$\tilde{\bU}_{0,(k)}=\{y_k,\bx_k,z_k\}$,  
$\hat{\btheta}_k=\tilde{U}_{H,(k)}$, and  $\pi_l(\cdot)$ 
represents a $d_l$-dimensional multivariate Gaussian distribution. 

\begin{remark} \label{cyclicgraph} It is interesting to point out that for the stochastic version of the EFI network, $(\bY_n,\tilde{\bU}_H,\ldots,\tilde{\bU}_1)$ forms a directed cyclic graph (DCG): $\bY_n\to \tilde{\bU}_1 \to \cdots\to \tilde{\bU}_H \to \bY_n$. See e.g. \cite{Sethuraman2023NODAGSFlowNC} for analysis of DCGs. For our case, $\bY_n$ is known, 
 which serves as intervention variables and greatly simplifies the problem. 
 Specifically,  as $\epsilon \downarrow 0$,  $\hat{\btheta}_k$'s can be uniquely determined via $p_{\epsilon}(\cdot)$ for a given set of $(\bY_n,\bX_n,\bZ_n)$ and, therefore, 
 (\ref{decompeq}) is reduced to the posterior distribution for a conventional StoNet with 
 $(y_k,\bx_k,z_k)$ serving as the input and $\hat{\btheta}_k$ serving as the target output. 
\end{remark}

\begin{assumption} \label{ass3}  
(i) The activation function $\Psi_l(\cdot)$ used for each hidden neuron is $c'$-Lipschitz continuous for some constant $c'$;   (ii) $d_l \log(d_l)  \prec n/\log(n)$ for $l=0,1,2,\ldots, H$, where $d_0$ 
denotes the dimension of $(Y,X,Z)$; (iii) the network's depth $H$ and widths $d_l$'s, for $l=0,1,2,\ldots,H$, are all allowed to increase with the sample size $n$.
\end{assumption}

Assumption \ref{ass3}-(i) has covered many commonly used activation functions such as 
{\it ReLU}, {\it tanh}, and {\it sigmoid}. 
Assumption \ref{ass3}-(ii) restricts the width of the DNN.
To model a dataset of dimension $d_0$, we generally need a model of dimension 
$dim(\btheta) \succeq d_0$. Since we do not aim to impose any sparsity constraints on $\btheta$ in our main theory, we assume 
 $d_0$ also satisfies the constraint $d_0 \log (d_0) \prec n$ such that 
$dim(\btheta) \prec n$ is still possible to hold.  



Suppose Assumption \ref{ass1} and Assumption \ref{ass3} hold. Following \cite{LiangSLiang2022} and  \cite{SunLiang2022kernel},
we can establish the existence of  
a small value $\tau(d_1,d_2,\ldots,d_H)$,  as a function of $d_1,d_2,\ldots,d_H$, 
such that if $\max\{\sigma_1, \sigma_2, \ldots, \sigma_H\} \prec \tau(d_1,d_2,\ldots,d_H)$ and $\epsilon \to \infty$, then 
\begin{equation} \label{equivLikelihoodeq1}
\begin{split}
\sup_{\bw_n\in \mathcal{W}_n} & \frac{1}{n}\Big| \log \pi_{\epsilon}(\bw_n|\bY_n, \tilde{\bU}_H, \ldots, 
\tilde{\bU}_1,\bX_n, \bZ_n) 
- \log \pi_{\epsilon}(\bw_n|\bY_n,\bX_n,\bZ_n) \Big| \stackrel{p}{\to} 0,   \quad 
\mbox{as $n\to \infty$}.
\end{split}
\end{equation}
In other words, the DNN model in the EFI network and stochastic DNN model introduced above have asymptotically the same loss function as long as $\sigma_1, \ldots,\sigma_H$ are sufficiently small and $\epsilon$ is sufficiently small. 

Suppose that we want to estimate 
$\bw_n$ by maximizing the posterior distribution of the pseudo-complete data, i.e., 
\begin{equation} \label{misseq1}
\begin{split}
\hat{\bw}_{n}^{u} & =\arg\max_{\bw_n}\Big\{\log \pi_{\epsilon} (\bw_n|\bY_n,\tilde{\bU}_H,\ldots, \tilde{\bU}_1,\bX_n, \bZ_n) \Big\}.
\end{split}
\end{equation} 
In the next subsection, we establish the consistency of $\hat{\bw}_n^u$, as an estimator 
of $\tilde{\bw}_n^*$, under appropriate conditions. Note that by Assumptions \ref{ass1}-\ref{ass2} and (\ref{equivLikelihoodeq1}), we  have 
\begin{equation} \label{thetaequiv1}
\|\hat{\bw}_n^u-\hat{\bw}_n^*\|\stackrel{p}{\to} 0, \quad \mbox{as $n\to\infty$}.
\end{equation}
 

  
\subsubsection{Consistency of the Sparse DNN Model Estimation} \label{IROsparsedecodersection}

This section gives a constructive proof for the consistency of $\hat{\bw}_n^u$ 
based on the IRO algorithm \citep{liang2018imputation}. 
To solve the optimization problem in (\ref{misseq1}), 
the IRO algorithm starts with an initial weight setting $\hat{\bw}_n^{(0)}$ 
and then iterates between the {\it imputation} and {\it regularized-optimization} steps: 
\begin{itemize}
    \item {\bf Imputation:}  For each block, conditioned on the current parameter estimate $\hat{\bw}_n^{(t)}$, simulate the latent variables  $(\tilde{\bU}_H^{(t+1)},\ldots, \tilde{\bU}_1^{(t+1)})$ 
     from the predictive  distribution 
\begin{equation} \label{decompeq2}
\begin{split} 
 \pi(\tilde{\bU}_H^{(t+1)},\ldots, \tilde{\bU}_1^{(t+1)}|\bY_n,\bX_n,\bZ_n,\hat{\bw}_n^{(t)}) & \propto 
 \prod_{k=1}^n p_{\epsilon}(y_k|\bx_k,z_k,\bar{\btheta}^{(t+1)}) \\  
  & \times \prod_{l=1}^H \prod_{k=1}^n \pi_l(\tilde{U}_{l,(k)}^{(t+1)}| \tilde{U}_{l-1,(k)}^{(t+1)},\hat{\bw}_{n,l}^{(t)}), \\
 \end{split} 
\end{equation}
 where $t$ indexes iterations, $\bar{\btheta}^{(t+1)}=\frac{1}{n} \sum_{k=1}^n \tilde{U}_{H,(k)}^{(t+1)}$, and $\hat{\bw}_{n,l}^{(t)}$ denotes the component of 
 $\hat{\bw}_{n}^{(t)}$ corresponding to the weights at the $l$-th layer. 
 
  \item {\bf Regularized-optimization:} Given the pseudo-complete data $\{ \tilde{\bU}_H^{(t+1)},\ldots, \tilde{\bU}_1^{(t+1)}$, $\bY_n$, $\bX_n$, 
   $\bZ_n\}$, 
  update  $\hat{\bw}_n^{(t)}$ by maximizing the penalized log-likelihood function: 
  \begin{equation} \label{IROsolution}
  \begin{split}
  \hat{\bw}_n^{(t+1)} & =\arg\max_{\bw_n}\Big\{ \log \pi(\bY_n, \tilde{\bU}_H^{(t+1)},\ldots, \tilde{\bU}_1^{(t+1)}|\bX_n,\bZ_n,\bw_n) +\log \pi(\bw_n) \Big\}, 
 \end{split}
\end{equation}
 which, by the decomposition (\ref{decompeq}), can be reduced to solving 
 for  $\hat{\bw}_{n,1}^{(t+1)}, \ldots, \hat{\bw}_{n,H}^{(t+1)}$, separately. 
 The penalty function $\log \pi(\bw_n)$ should be chosen such that $\hat{\bw}_n^{(t+1)}$ forms a consistent estimator for the working true parameter 
 \begin{equation} \label{mappingeq}
 \begin{split}
 &\bw_{n,*}^{(t+1)} =\arg\max_{\bw_n} \mathbb{E}_{\hat{\bw}_n^{(t)}} 
 \log \pi(\bY_n, \tilde{\bU}_H^{(t+1)},\ldots, \tilde{\bU}_1^{(t+1)}|\bX_n,\bZ_n,\bw_n) \\
 & = \arg\max_{\bw_n}   
 \int \log \pi(\bY_n, \tilde{\bU}_H^{(t+1)},\ldots, \tilde{\bU}_1^{(t+1)}|\bX_n,\bZ_n,\bw_n) \\ 
 &  \quad \times \pi( \tilde{\bU}_H^{(t+1)},\ldots, \tilde{\bU}_1^{(t+1)}|\bY_n,\bX_n,\bZ_n,\hat{\bw}_n^{(t)}) \pi_{\epsilon}(\bY_n|\bX_n, \bZ_n,\tilde{\bw}_n^*) 
 d\tilde{\bU}_H^{(t+1)} \cdots d\tilde{\bU}_1^{(t+1)} d\bY_n,
 \end{split}
 \end{equation}
 where the likelihood function of the latent variables $(\tilde{\bU}_H^{(t+1)},\ldots, \tilde{\bU}_1^{(t+1)})$ is evaluated at $\hat{\bw}_n^{(t)}$, and  
 $\tilde{\bw}_n^*$ corresponds to  the true parameters of the underlying sparse DNN model. 
 \end{itemize}

To prove the consistency of $\hat{\bw}_n^{(t)}$ as $n$ and $t$ approach to infinity,
we need Assumptions \ref{ass4}-\ref{ass6} as specified below. 
For a matrix $\bSigma$, we define the $m$-sparse minimal eigenvalues as follows: 
\[
\phi_{\rm \min}(m|\bSigma)=\min_{\bbeta: \|\bbeta\|_0 \leq m} \frac{\bbeta^{\top} \bSigma \bbeta}{\bbeta^{\top} \bbeta},
\]
which represents the minimal eigenvalues of any $m\times m$-dimensional principal submatrix, and $\|\bbeta\|_0$ denotes the number of nonzero elements in $\bbeta$.
Let $\bSigma_{l,i}^{(t)} \in \mathbb{R}^{d_{l-1} \times d_{l-1}}$ denote the sample covariance matrix of the input variables for the regression formed 
for neuron $i$ of layer $l$ at iteration $t$, and 
let $s_{l,i}^{(t)}$ denote the size of the true regression model 
as implied by the working true parameter $\bw_{n,*}^{(t)}$. 

\begin{assumption} \label{ass4} 
 (i) All the variables $\{\bX_n,\bY_n,\bZ_n^{(t)}\}$ are uniformly bounded, where $\bZ_n^{(t)}$ denotes the imputed values of $\bZ_n$ at iteration $t$; 
(ii) there exist some constants $\kappa_{0,1}$ and 
$s_{1,i}^{(t)} \leq \tilde{s}_{1,i}^{(t)} \leq \min\{d_{0},n\}$ 
such that $\phi_{\rm \min}(\tilde{s}_{1,i}^{(t)}|\bSigma_{1,i}^{(t)}) \geq \kappa_{0,1}$ holds uniformly for any neuron $i \in\{1,2,\ldots,d_1\}$ and any  
iteration $t\in \{1,2,\ldots,T\}$, where $\bSigma_{1,i}^{(t)}$ denotes the covariance matrix of $\{\bX_n,\bY_n,\bZ_n^{(t)}\}$.
\end{assumption}

Assumption \ref{ass4}-(i) restricts the pseudo-complete data $(\bx, \by, \bz)$ to be 
uniformly bounded. To satisfy this condition, we can add a data transformation/normalization layer to the DNN model, ensuring that the transformed input values fall within the bounded set.
Specifically, the transformation/normalization layer can form a bijective mapping and contain no tuning parameters. For example, when dealing with standard Gaussian random variables, we can transform them to be uniform over (0,1) via the probability integral transformation $\Phi(\cdot)$, the cumulative distribution function (CDF) of the standard Gaussian random variable.  


Assumption \ref{ass4}-(ii) is natural for the problem. As implied by Assumption \ref{ass3}-(ii), the upper bound $s_{1,i}^{(t)} \leq \tilde{s}_{1,i}^{(t)} \leq \min\{d_{0},n\}$ 
can always hold. For other layers, the upper bound is also true by Assumption \ref{ass3}-(ii), and the sparse eigenvalue property can be directly established, see Lemma \ref{lemma:mineigen} below.

\begin{lemma} \label{lemma:eigen} 
Consider a random matrix $\bU\in \mathbb{R}^{n\times d}$ with $n \geq d$.  
Suppose that 
the eigenvalues of $\bU^{\top}\bU$ are upper bounded, i.e.,  
$\lambda_{\max}(\bU^{\top}\bU) \leq n \kappa_{\max}$ for some constant $\kappa_{\max}>0$.  
Let $\Psi(\bU)$ denote an elementwise transformation of $\bU$. Then 
\begin{equation} \label{PsiLemma}
\lambda_{\max}\left( (\Psi(\bU))^{\top} (\Psi(\bU)) \right) \leq n \kappa_{\max},
\end{equation}
for the {\it tanh}, {\it sigmoid} and ReLU transformations. 
\end{lemma} 
\begin{proof}
For ReLU, (\ref{PsiLemma}) follows from Lemma 5 of \cite{Dittmer2018SingularVF}. For {\it tanh} and {\it sigmoid}, since they are Lipschitz continuous with a Lipschitz constant of 1, Lemma 5 of  \cite{Dittmer2018SingularVF} also applies. 
\end{proof}

\begin{lemma} \label{lemma:mineigen} Consider an auxiliary stochastic neural network 
as defined in (\ref{eq:stochastic}) with an activation function {\it tanh, sigmoid}, or {\it ReLU}. Then for any  any layer $l \in \{2,3,\ldots,H\}$, neuron $i \in\{1,2,\ldots,d_l\}$, and 
iteration $t\in \{1,2,\ldots,T\}$, there exists a number $\tilde{s}_{l,i}^{(t)}$ such that 
$s_{l,i}^{(t)} \leq \tilde{s}_{l,i}^{(t)} \leq \min\{d_{l-1},n\}$ and $\phi_{\rm \min}(\tilde{s}_{l,i}^{(t)}|\bSigma_{l,i}^{(t)}) \geq  \kappa_{0,l}$ hold. 
\end{lemma}
\begin{proof}
We use $\tilde{U}_l \in \mathbb{R}^{d_l}$ and $V_l \in \mathbb{R}^{d_l}$ to denote generic vectors corresponding to the $l$-th layer, where $\tilde{U}_k=V_l+\be_l$. 
 Additionally, we use $\tilde{\bU}_l \in \mathbb{R}^{n\times d_l}$ and $\bV_l \in \mathbb{R}^{n\times d_l}$ to denote the matrices (for all observations) corresponding to the $l$-th layer. 
For the auxiliary  stochastic neural network, since $\sigma_l^2$'s have been set to very small values,  it follows from (\ref{eq:stochastic}) that 
\[
\Psi(\tilde{U}_{l}) \approx \Psi(V_{l}) + \nabla_{V_{l}}\Psi(V_{l}) \circ \be_l, \quad l=1,2,\ldots,H-1,
\]
where $\circ$ denotes elementwise product, 
$V_l=(v_{l,1},v_{l,2},\ldots, v_{l,d_l})^{\top}$, 
and $\be_l=(e_{l,1},e_{l,2},\ldots,e_{l,d_l})^{\top}$.
Then, for any $i \in \{1,2,\ldots, d_{l+1}\}$, 
\begin{equation}
\label{recursive_covariance}
\begin{split}
\bSigma_{l+1,i} &\approx \Var(\mathbb{E} (\Psi(V_l) + \nabla_{V_l}\Psi(V_l) \circ \be_l | V_l ) ) + \mathbb{E}(\Var(\Psi(V_l) + \nabla_{V_l}\Psi(V_l) \circ \be_l | V_l)) \\
& = \Var(\Psi(V_l)) + \diag\left\{\sigma_{l}^2 \mathbb{E}[\nabla_{V_l}\Psi(V_l) \circ \nabla_{V_l}\Psi(V_l)] \right\}, \\
\end{split}    
\end{equation}
where $\diag\{\ba\}$, $\ba\in \mathbb{R}^{d_l}$, denotes a $d_l \times d_l$ diagonal matrix with the diagonal elements given by the vector $\ba$.

By induction, building on Assumption \ref{ass4}-(i) and 
Lemma \ref{lemma:eigen},
we can assume that the maximum eigenvalue of $(\Psi(\tilde{\bU}_{l-1}))^{\top} \Psi (\tilde{\bU}_{l-1})$ is upper bounded by $n \kappa_{\max}$.
By an extension of Ostrowski's theorem, see Theorem 3.2 of \cite{Higham1998ModifyingTI},
we have 
\[
\begin{split}
\lambda_{\max}(\bV_l^T \bV_l) & = \max_{\|\bu\|=1} \bu^{\top} \bw_{n,l} (\Psi(\tilde{\bU}_{l-1}))^{\top} \Psi (\tilde{\bU}_{l-1}) \bw_{n,l}^{\top} \bu \\
&  \leq 
\lambda_{\max}( (\Psi(\tilde{\bU}_{l-1}))^{\top} \Psi (\tilde{\bU}_{l-1})) \max_{\|\bu\|=1} \bu^{\top} \bw_{n,l} \bw_{n,l}^{\top}\bu\\
& =n\kappa_{\max} \tau_{\max}, \\
\end{split}
\]
where the existence of the upper bound $\lambda_{\max}(\bw_{n,l}^{\top} \bw_{n,l}) \leq \tau_{\max}$ follows from the boundedness of $\bw_n$ as implied by Assumption \ref{ass1}-(i). 
By choosing $\bu$ as a one-hot vector, it is easy to see that for any $i \in \{1,2,\ldots,d_l\}$, 
\begin{equation} \label{eigenboundeq}
 \sum_{j=1}^n V_{l,i,j}^2 \leq \lambda_{\max}(\bV_l^{\top} \bV_l) \leq n \kappa_{\max} \tau_{\max}, 
\end{equation}
where $V_{l,i,j}$ denotes the $(j,i)$th element of $\bV_l$. This further implies, as $n\to \infty$, 
\begin{equation} \label{meanVeq}
 \mathbb{E} (V_{l,i,j}^2)  \leq \kappa_{\max} \tau_{\max}. 
\end{equation}
In words, $V_{l,i,j}$ has bounded mean and variance.  

By Markov's inequality, we can bound $V_{l,i,j}$ to a closed interval with a high probability, i.e., $P(|V_{l,i,j}|\leq C)\geq 1- \kappa_{\max} \tau_{\max}/C^2$ for some large constant $C$. Therefore, for any  activation function which has nonzero gradients on any closed interval, e.g., {\it tanh} and {\it sigmoid}, there exists a constant $c>0$ such that 
\[
 \mathbb{E} [\nabla_{V_{l,i,j}} \Psi(V_{l,i,j})]^2 \geq c.
 \]
which implies $\phi_{\min}(d_{l}|\bSigma_l) \geq c \sigma_l^2$.

For the ReLU activation function, if a hidden neuron belongs to the true neuron set at iteration $t$ (as determined by $\bw_{n,*}^{(t)}$), 
then $\Psi(V_{l,i,j}^{(t)})$ cannot be constantly 0 over all $n$ samples.
Therefore, it is reasonable to assume that there exists a 
threshold $q_{\rm min} \in (0,1)$ such that $\mathbb{E}[\nabla_{V_{l,i,j}} \Psi(V_{l,i,j}^{(t)})]^2 \geq q_{\rm min}$ for any true neuron in all iterations.
Under this assumption, we would at least have 
\[
\phi_{\rm min}(|\bs_l^{(t)}| | \bSigma_l) \geq \sigma_{l}^2 q_{\rm min},  \quad l=1,2,\ldots,h; \ \ t=1,2,\ldots, T, 
\]
where $\bs_l^{(t)} \subset \bS^{(t)}$ denotes the set of true neurons at layer $l$. 
\end{proof}

To study the property of the coefficient estimator for each regression formed 
in the auxiliary stochastic neural network, we introduce the following 
lemma, which is a restatement of Theorem 3.4 of \cite{SongLiang2021Nearly}. 

\begin{lemma} \label{lemma:linearR}
(Theorem 3.4; \cite{SongLiang2021Nearly}) Consider a linear regression 
\[
\by=\bX \bbeta+\sigma \bepsilon,
\]
 where $\by \in \mathbb{R}^n$, $\bX\in \mathbb{R}^{n\times p_n}$, $\bbeta\in \mathbb{R}^{p_n}$, 
 $\sigma>0$, and $\bepsilon \sim \mathcal{N}(0,I_{p_n})$ is Gaussian noise.  
 Suppose that the model satisfies the following conditions: 
 \begin{itemize} 
 \item[(A1)] (i) All the covariates are uniformly bounded; (ii) the dimensionality can be high with $p_n \geq n$; and (iii) there exists some integer $\bar{p}$ (depending on $n$ and $p_n$) and a fixed constant $\kappa_{0}$ such that $\bar{p} \succ s_n$  and 
 $\lambda_{\rm \min}(\bX_{\xi}^{\top} \bX_{\xi}) \geq n \kappa_{0}$ for any subset model 
  $|\xi| \leq \bar{p}$, where $s_n$ denotes the size of the  true model $\xi^*$, 
   and $\lambda_{\rm \min}(\cdot)$ denotes the minimum eigenvalue of a square matrix. Let $\hat{\bbeta}_{\xi}$ denote the MLE of the true model. 
   
  \item[(A2)] (i) $s_n \log p_n \prec n$; (ii) $\max\{ |\beta_j^*/\sigma^*|: j=1,2,\ldots,p_n \} \leq \gamma_3 E_n$, where  $\beta_j^*$'s and $\sigma^*$ denote the 
  true parameter values of the regression model,  $\gamma_3 \in (0,1)$ is a fixed constant,  and $E_n$ is nondecreasing with respect to $n$. 

 \item[(A3)] Let each component of $\bbeta$ be subject to the following mixture Gaussian prior distribution 
 \[
 \beta_j/\sigma^* \sim (1-\rho) \mathcal{N}(0, \tilde{\sigma}_0^2)+ \rho \mathcal{N}(0,\tilde{\sigma}_1^2), \quad j=1,2,\ldots, p_n, 
 \]
 where $E_n/\tilde{\sigma}_1^2+\log \tilde{\sigma}_1 \asymp \log p_n$, 
 $\rho=1/p_n^{1+u}$, and $\tilde{\sigma}_0 \leq a_n/\sqrt{2(1+u) \log p_n}$ 
  for some constant $u>1$ and sequence $a_n \prec \sqrt{1/(n s_n \log p_n)}/p_n$.   
 Additionally, $s_n E_n \sqrt{s_n \log p_n/n}$  $\prec \tilde{\sigma}_1^2$ and 
 $\min_{j \in \xi^*} |\beta_j^*| \geq M_1 \sqrt{\log p_n/n}$ for some sufficiently large 
 $M_1>0$. 
 \end{itemize} 
 Let $\hat{\bbeta}_n$ denote the MAP 
 estimator of $\bbeta$. Then there exists a constant $c$ such that 
 \[
  \mathbb{E} \|\hat{\bbeta}_n-\bbeta^*\|^2 =c (\sigma^*)^2 \|\bX_{\xi^*}^{\top} \bX_{\xi^*}\|^{-1} 
  \leq \frac{c (\sigma^*)^2}{n \kappa_{0}},
 \]
Specifically, with dominating probability, we have $\hat{\bbeta}_{n,j}=\hat{\bbeta}_{\xi^*,j'}$ if $j\in \xi^*$ and 0 otherwise, where $\hat{\bbeta}_{n,j}$ denotes the $j$th element of $\hat{\bbeta}_n$, and   $\hat{\bbeta}_{\xi^*,j'}$ denotes the element 
 of $\hat{\bbeta}_{\xi^*}$ that  corresponds to 
 $\bbeta_{n,j}$.
 \end{lemma}

It is easy to see that under Assumption \ref{ass1}, Assumption \ref{ass3}, 
and Assumption \ref{ass4}, 
each linear regression formed in the stochastic neural network satisfies 
conditions (A1) and (A2) of Lemma \ref{lemma:linearR}. In particular, we can set 
$E_n$ as the diameter of $\mathcal{W}_n$. 
To ensure the condition (A3) to be satisfied, we make the following assumption: 

\begin{assumption} \label{ass5} Let each connection weight $w_{n,l,i,j}$ be subject 
to the following mixture Gaussian priro distribution:
\[
 w_{n,l,i,j}/\sigma_l \sim (1-\rho) \mathcal{N}(0, \tilde{\sigma}_{0,l}^2)+ \rho \mathcal{N}(0,\tilde{\sigma}_{1,l}^2), \ \ l=1,2,\ldots,H, \ \  i=1,2,\ldots, d_l, \ \  j=1,2,\ldots, d_{l-1}, 
 \]
 where we set $\rho$, $\tilde{\sigma}_{0,l}$ and $\tilde{\sigma}_{1,l}$ such that  
 $\rho=1/d_{l-1}^{1+u}$, 
$\tilde{\sigma}_{0,l} \prec 1/(d_{l-1}^{3/2} \log(d_{l-1}) \sqrt{2(1+u)n})$, and 
$E_n/\tilde{\sigma}_{1,l}^2+\log \tilde{\sigma}_{1,l} \asymp \log d_{l-1}$ 
for some constant $u>1$ and any $l=1,2,\ldots,H$. 
Additionally, there exists some constant $M_1>0$ such that 
 $\min_{1 \leq i \leq d_l, j \in \xi_i^*} |w_{n,l,i,j}^*| \geq M_1 \sqrt{\log d_{l-1}/n}$.
\end{assumption}

It is easy to verify that the condition (A3) holds under Assumption \ref{ass3}-(ii) and Assumption \ref{ass5}. In addition, the $w$-min condition, i.e., 
 $\min_{1 \leq i \leq d_l, j \in \xi_i^*} |w_{n,l,i,j}^*| \geq M_1 \sqrt{\log d_{l-1}/n}$, is rather weak and can be generally satisfied as $n$ becomes large.

\begin{theorem} \label{lemma:partI} Suppose that a mixture Gaussian penalty is imposed on the weights of the DNN model in the EFI network and Assumptions 
\ref{ass1} and \ref{ass3}-\ref{ass5} hold. Furthermore, suppose 
$\sum_{l=1}^H d_l \sigma_l^2/\sigma_{l-1}^2 \prec n$ holds, where 
$\sigma_0=O(1)$ represents a constant.  
Then there exist a constant $c$ such that 
\[
E\|\hat{\bw}_n^{(t)} -\bw_{n,*}^{(t)}\|^2 \leq \frac{c}{n} \sum_{l=1}^H d_l \frac{\sigma_l^2}{\sigma_{l-1}^2} :=r_n \prec o(1). 
\]
\end{theorem}

\begin{proof}
By the above analysis, each regression formed in the stochastic neural network 
satisfies the conditions of Lemma \ref{lemma:linearR}. Therefore, 
the sparse eigenvalue lower bounds established in Lemma \ref{lemma:mineigen}  hold for the stochastic neural network.  Further, 
 by summarizing the $l_2$-errors of coefficient estimation for all $\sum_{l=1}^{H} d_l$ linear regressions, we can conclude the proof.  
\end{proof} 

 It is important to note that the condition $\sum_{l=1}^H d_l \sigma_l^2/\sigma_{l-1}^2 \prec n$ allows the width of each layer of the neural network 
 to increase with $n$ at a rate  as high as $O(n/\log(n))$. 
 This accommodates the scenarios where  
  $dim(\btheta)=O(n^{\zeta})$ for some $\frac{1}{2} \leq  \zeta <1$. 
 In this case, we have $d_H=dim(\btheta)$ for the DNN model in the EFI network, which enables the uncertainty of $\btheta$ to be properly quantified as implied by Theorem \ref{thm:largemodel} proved below.

Further, let's consider the mapping $M(\bw_n)$ as defined in (\ref{mappingeq}), i.e., 
\[
M(\bw_n)=\arg\max_{\bw_n^{\prime}}  \mathbb{E}_{\bw_n} \log\pi(\bY_n,\tilde{\bU}_H,\ldots,\tilde{\bU}_1 |\bX_n, \bZ_n,\bw_n^{\prime}).
\]
As argued in \cite{liang2018imputation} and \cite{Nielsen2000}, 
it is reasonable to assume that 
the mapping is contractive. A recursive application of the mapping, i.e., setting
$\hat{\bw}_n^{(t+1)}=\bw_{n,*}^{(t+1)}=M(\hat{\bw}_n^{(t)})$, leads to a monotone increase of the target expectations 
\[
\mathbb{E}_{\hat{\bw}_n^{(t)}} \log \pi(\bY_n, \tilde{\bU}_H^{(t+1)},\ldots,\tilde{\bU}_1^{(t+1)}|\bX_n,\bZ_n,\hat{\bw}_n^{(t+1)})
\]
for $t=1,2,\ldots,T$. 

\begin{assumption} \label{ass6}  The mapping $M(\bw_n)$ is differentiable. Let $\lambda_{\rm \max}(M_{\bw_n})$ be the largest singular value of $\partial M(\bw_n)/\partial \bw_n$. There
exists a number $\lambda^* <1$ such that  $\lambda_{\rm \max} (M_{\bw_n})\leq \lambda^*$ for all 
$\bw_n \in \mathcal{W}_n$ for sufficiently large $n$ and almost every training dataset $D_n$.
\end{assumption}

 \begin{theorem} \label{lemma:partII} Suppose that a mixture Gaussian penalty is imposed on the weights of the DNN model in the EFI network; Assumptions 
\ref{ass1} and \ref{ass3}-\ref{ass6} hold; $\epsilon$ is sufficiently small; and 
$\sum_{l=1}^H d_l \sigma_l^2/\sigma_{l-1}^2$ $\prec n$ holds, 
where $\max\{\sigma_1, \sigma_2, \ldots, \sigma_H\} \prec \tau(d_1,d_2,\ldots,d_H)$ as defined in Section \ref{auxNN} and $\sigma_0=O(1)$ represents a constant.  
Then $\|\hat{\bw}_n^{(t)} -\tilde{\bw}_n^*\| \stackrel{p}{\to} 0$ for sufficiently large $n$ and sufficiently large $t$ and almost every training dataset $D_n$. 
\end{theorem}
\begin{proof}
This lemma directly follows from Theorem 4 of \cite{liang2018imputation} that the estimator $\hat{\bw}_n^{(t)}$ is consistent when both $n$ and $t$ are sufficiently large.
\end{proof}

In summary, we have given a constructive proof for the consistency of sparse 
stochastic neural network based on the sparse learning theory 
developed for high-dimensional linear regression. 
Our proof implies that under the conditions of Theorem \ref{lemma:partII},
such a consistent estimator can also be obtained by a direct 
calculation of $\hat{\bw}_n^u$ as defined in (\ref{misseq1}). 
Furthermore, it follows from (\ref{thetaequiv1}) that 
\begin{equation} \label{thetaequiv2}
\|\hat{\bw}_n^* -\tilde{\bw}_n^*\| \stackrel{p}{\to} 0, \quad \mbox{as $n \to \infty$.}
\end{equation}

\paragraph{Proof of Theorem \ref{thm:largemodel}} 
\begin{proof} As a summary of Lemma \ref{lemma:equivalent} and (\ref{thetaequiv2}),  we have
\[
\|\bw_n^* -\tilde{\bw}_n^*\| \stackrel{p}{\to} 0, \quad \mbox{as $n\to \infty$},
\]
by setting $\sigma_1=\sigma_2=\cdots=\sigma_H \prec \tau(d_1,d_2,\ldots,d_H)$. 
Subsequently, $\hat{g}(y,\bx,z,\bw_{n}^*)$ constitutes a 
consistent estimator of $\btheta$, and so does $G^*(\bY_n,\bX_n,\bZ_n)$ following the arguments 
in Section \ref{Sect:outline}. 
\end{proof}

In summary, through the introduction of an auxiliary stochastic neural network model and the utilization of the convergence theory of the IRO algorithm, 
we have justified the consistency of the sparse DNN model under mild conditions.

\begin{remark}
The result presented in Theorem \ref{thm:largemodel} is interesting: we can set the width of each layer of the DNN model to be of order $O(n^{\alpha})$ for some $1/2\leq \alpha<1$ 
and set its depth to be $O(n^{\alpha'})$ for some $0<\alpha'<1-\alpha$. Given the approximation ability of the sparse DNN model as studied in \cite{SunSLiang2021}, where it was proven that a neural network of size $O(n^{\tilde{\alpha}})$ 
for some $0<\tilde{\alpha}<1$ has been large enough for approximating many classes 
of functions, 
EFI can be used for uncertainty quantification for 
deep neural networks if their sizes are appropriately chosen. 
Specifically, for the Double-NN approach, we can set the size of the first neural network (for approximation of the inverse function $\btheta=G(\bY_n,\bX_n,\bZ_n)$) 
to be large under the constraint $\sum_{l=1}^H d_l \prec n$, and set the second
neural network (for approximation of the function $\bY_n=f(\bX_n,\bZ_n;\btheta)$) 
to be relatively small with the size $O(n^{\tilde{\alpha}})$ for some 
$0< \tilde{\alpha} <1$. By the theory developed in this paper, the uncertainty of the second neural network can still be correctly quantified using EFI. 
\end{remark}

\section{CQR Method} \label{CQR}

For a test point $\bx$ belonging to the class $\mathcal{I}_c$, 
CQR aims to find the intervals $\hat{C}_1(\bx)$ and $\hat{C}_{ITE}(\bx)$ such that
\[
P(Y(1;\bx) \in \hat{C}_1(\bx))\geq 1-\alpha, \quad 
P(Y(1;\bx)-Y^{obs}(0;\bx)\in \hat{C}_{ITE}(\bx)) \geq 1-\alpha.
\]
Similarly, for a test point $\bx$ belonging to the class 
$\mathcal{I}_t$, 
CQR aims to find the intervals $\hat{C}_0(\bx)$ and $\hat{C}_{ITE}(\bx)$ such that
\[
\begin{split}
P(Y(0;\bx) \in \hat{C}_0(\bx)) & \geq 1-\alpha, \quad 
P(Y^{obs}(1;\bx)-Y(0;\bx)\in \hat{C}_{ITE}(\bx)) \geq 1-\alpha.
\end{split}
\]
To achieve the above goals, CQR initially splits the training dataset $\mathcal{X}_{train}$ to $(\mathcal{X}_{train},\mathcal{X}_{valid})$. A quantile regression model, such as BART or random forest, is trained on $\mathcal{X}_{train}$. Denote the output of the quantile regression by  $[\hat{q}_{\alpha/2}(t,\bx),\hat{q}_{1-\alpha/2}(t,\bx)]$, and calculate the score $s_i(t)=max( \hat{q}_{\alpha/2}(t,\bx_i)-y_i(t), y_i(t)-\hat{q}_{1-\alpha/2}(t,\bx_i))$ for each $\bx_i\in \mathcal{X}_{valid}$. 
Note that $s_i$ acts like a residual derived from the validation set. Let $\hat{s}(t,\alpha)=Quantile(\{s_i(t)\}_{\bx_i\in \mathcal{X}_{valid}};[\frac{(n+1)\alpha ]}{n}])$. Then   $\hat{C}_t(\bx_j)$ given below is the conformal prediction interval  for any $\bx_j\in \mathcal{X}_{test}$: 
\[
\begin{split}
    \hat{C}_t(\bx_j)&=[\hat{q}_{\alpha/2}(t,\bx_j)-\hat{s}(t,1-\alpha),\hat{q}_{1-\alpha/2}(t,\bx_j)+\hat{s}(t,1-\alpha)] \\ 
        &=[\hat{Y}^L(t;\bx_j),\hat{Y}^R(t;\bx_j)] 
\end{split}
\] 
For a more refined approach, one can consider incorporating propensity scores as weights as discussed in \cite{Tib2019wconformal}.

Case (i) and Case (ii) described in Section \ref{ITEPsect} can be addressed by the above approach. However, for Case (iii) there, the conformal prediction needs further steps. First, construct a pair of prediction intervals at level $1-\alpha/2$; namely, $[\hat{Y}^L(1;\bx),\hat{Y}^R(1;\bx)]$ for $Y(1)$ and $[\hat{Y}^L(0;\bx),\hat{Y}^R(0;\bx)]$ for $Y(0)$.  Then, construct an interval for ITE as follows:
\[
\hat{C}^{naive}_{ITE}(\bx)=[\hat{Y}^L(1;\bx)-\hat{Y}^R(0;\bx), \hat{Y}^R(1;\bx) -\hat{Y}^L(0;\bx) ].
\]
We refer to this approach as the ``naive'' approach, which usually leads to very wide 
prediction intervals. 

Another option for case (iii) is the so-called ``nested'' approach by 
splitting $\mathcal{X}_{train}$ into two folds, denoted by $(\mathcal{X}_{train,1},\mathcal{X}_{train,2})$.
On the first fold, train $\hat{C}(1,\bx)$ and $\hat{C}(0,\bx)$ by applying 
conformal inference. On the second fold, 
for each $\bx_i\in \mathcal{X}_{train,2}$, compute  
\[
    \hat{C}(\bx_i) =\begin{cases} 
     [Y^{obs}(1;\bx_i)-\hat{Y}^R(0;\bx_i),Y^{obs}(1,\bx_i)-\hat{Y}^L(0;\bx_i)], &  \textit{if \ } T=1, \\
     [\hat{Y}^L(1;\bx_i)-Y^{obs}(0;\bx_i),\hat{Y}^R(1;\bx_i)-Y^{obs}(0;\bx_i)], &  \textit{if \ } T=0, \\
\end{cases}
\]
where $\hat{Y}^L$  and $\hat{Y}^R$ are estimated based on $\mathcal{X}_{train,1}$. Note that the conformal inference is also applicable for the data with interval outcomes. 
Applying the conformal inference method with interval outcomes on $(\bx_i,\hat{C}(\bx_i))$ for $\bx_i\in \mathcal{X}_{train,2}$,  yielding the interval $\hat{C}^{exact}_{ITE}(\bx)$. 
Applying the conformal inference twice results in a sparse utilization of data for training the regression model and, subsequently, wider prediction intervals. The inexact method involves fitting conditional quantiles of $\hat{C}^L$ and $\hat{C}^R$, yielding an interval $\tilde{C}_{ITE}^{inexact}(\bx)$. The inexact method does not guarantee the coverage rate. Refer to \cite{lei2021ite} for the detail.

\section{Experimental Settings} \label{parasetting}

To enforce a sparse DNN to be learned for the inverse function $g(\cdot)$, we impose 
the following mixture Gaussian prior on each element of $\bw_n$: 
\begin{equation}\label{sparseprior}
    \pi(w)\sim \rho  N(0,\sigma_{1}^2)+(1-\rho)N(0,\sigma_{0}^2),
\end{equation}
where $w$ represents a generic element of $\bw_n$ and, unless stated otherwise, 
we set $\rho=1e-2$, $\sigma_0=1e-2$ and $\sigma_1=1$.
The elements of $\bw_n$ are {\it a priori} independent.

For EFI, we employ SGHMC in latent variable sampling, i.e., we simulate $\bZ_n^{(k+1)}$ in the following formula: 
\[  
\begin{split}
\bV_n^{(k+1)} &= (1-\varpi) \bV_n^{(k)} +\upsilon_{k+1} \widehat{\nabla}_{\bZ_n} \log \pi_{\epsilon}(\bZ_n^{(k)}|\bX_n,\bY_n,\bw_n^{(k)}) +\sqrt{2 \varpi \tau \upsilon_{k+1}} \be^{(k+1)}, \\
\bZ_n^{(k+1)}&=\bZ_n^{(k)}+\bV_n^{(k+1)}, 
\end{split}
\]
where $\tau=1$, $0<\varpi\leq 1$ is the momentum parameter, $\be^{(k+1)} \sim N(0,I_{d_{\bz}})$, 
and $\upsilon_{k+1}$ is the learning rate.
It is worth noting that the algorithm is reduced to SGLD if we set $\varpi=1$.

In the simulations, we set the learning rate sequence $\{\upsilon_k: k=1,2,\ldots\}$ and the step size sequence $\{\gamma_k: k=1,2,\ldots\}$ in the forms:
\[
\upsilon_k=\frac{C_{\upsilon}}{c_{\upsilon}+k^{\alpha}}, \quad  
\gamma_k=\frac{C_{\gamma}}{c_{\gamma}+k^{\alpha}},
\]
for some constants $C_{\upsilon}>0$, $c_{\upsilon}>0$, $C_{\gamma}>0$ and $c_{\gamma}>0$, and $\alpha \in (0,1]$. 
The values of $C_{\upsilon}$, $c_{\upsilon}$, $C_{\gamma}$,$c_{\gamma}$ and $\alpha$ used in different experiments are given below.

\subsection{ATE}

From the point of view  of equation solving, the model (\ref{causalmodel1}) in the main text can also  be written as 
\begin{equation} \label{causalmodel2}
y_i=\tau^{\prime} T_i^{\prime} +\mu^{\prime}+\bx_i \bbeta+\sigma z_i, \quad i=1,2,\ldots,n,
\end{equation}
where $T_i^{\prime}\in \{-1,1\}$, $\tau^{\prime}=\tau/2$, and $\mu^{\prime}=\mu+\tau/2$. 
Let $\btheta=(\tau^{\prime},\mu^{\prime},\bbeta^{\top},\sigma)^{\top}$. Equation (\ref{causalmodel2}) can be solved under the standard framework of EFI.
Unless otherwise noted, all results in this paper are  based on solving 
this type of transformed equations.

We set $\alpha=1/7$ and $\varpi=0.1$.
For $n=250$, we set 
$(C_{\upsilon},c_{\upsilon},C_{\gamma},c_{\gamma})=($200000, 1000000, 54000, 1000000).
For $n=500$ and 1000, we set 
$(C_{\upsilon},c_{\upsilon},C_{\gamma},c_{\gamma})=(500000, 1000000$, 54000, 1000000).

We set $\mK=5000$ as the number of burn-in iterations, and set $M=50,000$ as the number of iterations used for fiducial sample collection. We thinned the Markov chain by a factor of $B=5$; that is, we collected $M/B=10,000$ samples in each run. 

We set $\eta=500$ and $\epsilon=1/10$ in construction of the energy function, and perform gradient clipping by norm with 5000 during the first 
100 iterations (for finding a reasonably good initial point).  

 
The DNN in the EFI network has two hidden layers,  with the widths given by $d_1=90$ and $d_2=30$, respectively.

\subsection{Linear control response with Non-linear treatment response}

\begin{figure}[!h]
    \centering
    \includegraphics[width=0.6\textwidth]{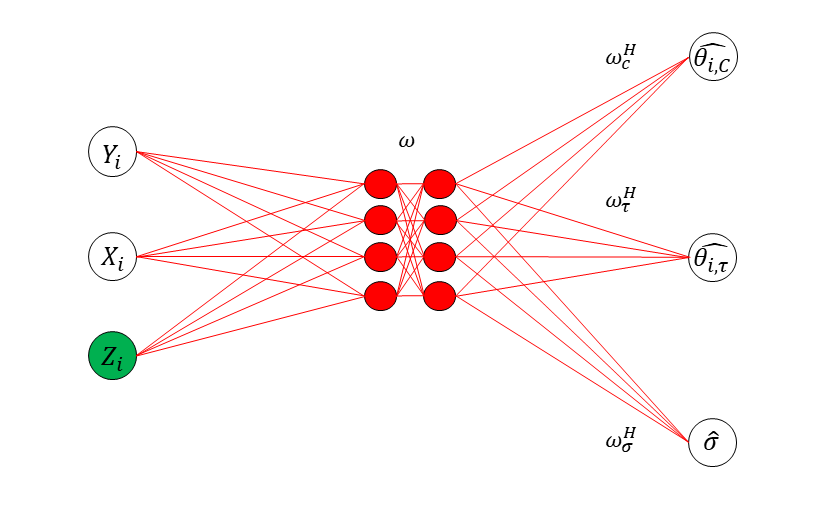}
    \caption{Description of the structure of the DNN model used in the Double-NN method: $\btheta$ (i.e., the output layer of the DNN model) can be partitioned into  three parts, namely, $\btheta_c$, $\btheta_{\tau}$, and $\sigma$, 
    where $\omega_{c}^L$, $\omega_{\tau}^L$ and $\omega_{\sigma}^L$ correspond to the parameters for $\btheta_c$, $\btheta_{\tau}$, and $\log(\sigma)$, respectively. 
    In the case that $\btheta_c$ or $\btheta_{\tau}$ represents a neural network, 
    we re-scale them by dividing a factor of 25 such that the output values
    of each output neuron are close to each other, easing the training process.   
    }
    \label{fig:double-NN} 
\end{figure}
 
We set $\alpha=1/7$, $\varpi=0.1$, and 
$(C_{\upsilon},c_{\upsilon})=(200000,1000000)$ and $(C_{\gamma},c_{\gamma})=(20,20000)$ for $\omega^{H}_{\tau}$ and $(C_{\gamma},c_{\gamma})=(20000,200000)$ for all other parameters.  Refer to Figure \ref{fig:double-NN} for the definition of $\omega_{\tau}^H$.

For initialization, we update the $\bw_n$ with randomly sampled $\bZ_n^{(t+1)}\sim \pi_0^{\otimes n}(\bZ_n)$ for the first 5000 iterations. We set $\mK=20,000$ as the number of burn-in iterations, and set $M=50,000$ as the number of iterations used for fiducial sample collection. We thinned the Markov chain by a factor of $B=5$; that is, we collected $M/B=10,000$ samples in each run. We set $\eta=10$ and $\epsilon=1/10$ in construction of the energy function, and perform gradient clipping by norm with 5000 during the first 100 iterations. 

 The DNN in the EFI network has two hidden layers,  with the widths given by $d_1=90$ and $d_2=30$, respectively. The DNN used for modeling the treatment effect has
 two hidden layers, each hidden layer consisting of 10 hidden neurons.

\subsection{Non-linear control response with Non-linear treatment response}

We set $\alpha=1/7$, $\varpi=0.1$, and 
$(C_{\upsilon},c_{\upsilon})=(500000,1000000)$ and $(C_{\gamma},c_{\gamma})=(2.5,1000000)$ for $\omega^{H}_{\tau}$ and $\omega_{c}^H$, and set $(C_{\gamma},c_{\gamma})=(20000,200000)$ for all other parameters. Refer to  Figure \ref{fig:double-NN} for the definitions of $\omega_{\tau}^H$ and $\omega_c^H$.
 
 For initialization, we update the $\bw_n$ with randomly sampled $\bZ_n^{(t+1)}\sim \pi_0^{\otimes n}(\bz_n)$ for the first 5000 iterations. We set $\mK=20,000$ as the number of burn-in iterations, and set $M=50,000$ as the number of iterations used for fiducial sample collection. We thinned the Markov chain by a factor of $B=5$; that is, we collected $M/B=10,000$ samples in each run. We set $\eta=10$ and $\epsilon=1/10$ in construction of the energy function, and perform gradient clipping by norm with 5000 during the first 100 iterations.

 The DNN in the EFI network has two hidden layers,  with the widths given by $d_1=90$ and $d_2=30$, respectively. The DNN used for modeling the treatment effect has two hidden layers, each hidden layer consisting of 10 hidden neurons. 
 The DNN used for modeling the response function under the control has two hidden layers, each hidden layer consisting of 10 hidden neurons.




 


\subsection{Lalonde}

We set $\alpha=1/7$, $\varpi=0.1$, and 
$(C_{\upsilon},c_{\upsilon})=(500000,1000000)$ and $(C_{\gamma},c_{\gamma})=(2.5,1000000)$ for $\omega^{H}_{\tau}$ and $\omega^{H}_{c}$, and $(C_{\gamma},c_{\gamma})=(1000,1000000)$ for all other parameters. 
Refer to  Figure \ref{fig:double-NN} for the definitions of $\omega_{\tau}^H$ and $\omega_c^H$.

 For initialization, we update the $\bw_n$ with randomly sampled $\bZ_n^{(t+1)}\sim \pi_0^{\otimes n}(\bz_n)$ for the first 10,000 iterations. We set $\mK=20,000$ as the number of burn-in iterations and set $M=50,000$ as the number of iterations used for fiducial sample collection. We thinned the Markov chain by a factor of $B=5$; that is, we collected $M/B=10,000$ samples in each run. We set $\eta=10$ and $\epsilon=1/10$ in construction of the energy function, and perform gradient clipping by norm with 5000 for first 100 iterations.

 The DNN in the EFI network has two hidden layers,  with the widths given by $d_1=90$ and $d_2=30$, respectively. The DNN used for modeling the treatment effect contains two hidden layers, each hidden layer 
 containing 10 hidden neurons. 
 The DNN used for modeling the response function under the control contains two hidden layers,  each hidden layer 
 containing 10 hidden neurons.


\subsection{NLSM}

We set $\alpha=1/7$, $\varpi=0.1$, and 
$(C_{\upsilon},c_{\upsilon})=(500000,1000000)$ and $(C_{\gamma},c_{\gamma})=(2.5,1000000)$ for $\omega^{H}_{\tau}$ and $\omega^{H}_{c}$, and $(C_{\gamma},c_{\gamma})=(5000,1000000)$ for all other parameters. Refer to  Figure \ref{fig:double-NN} for the definitions of $\omega_{\tau}^H$ and $\omega_c^H$.

 For initialization, we update the $\bw_n$ with randomly sampled $Z_n^{(t+1)}\sim \pi_0^{\otimes n}(\bz_n)$ for the first 10,000 iterations.  We set $\mK=20,000$ as the number of burn-in iterations and set $M=50,000$ as the number of iterations used for fiducial sample collection. We thinned the Markov chain by a factor of $B=5$; that is, we collected $M/B=10,000$ samples in each run. We set $\eta=10$ and $\epsilon=1/10$ in construction of the energy function, and perform gradient clipping by norm with 5000 for first 100 iterations.

 The DNN in the EFI network contains two hidden layers,  with the widths given by $d_1=90$ and $d_2=30$, respectively. The DNN used for modeling the treatment effect contains two hidden layers,  each hidden layer 
 containing 10 hidden neurons. 
 The DNN used for modeling the response function under control contains two hidden layers, each hidden layer 
 containing 10 hidden neurons.





 

\bibliographystyle{asa}
\bibliography{reference}

\end{document}